\documentclass{article}
\usepackage{iclr2027_conference,times}

\usepackage{amsmath,amsfonts,bm}

\def\eqref#1{equation~\ref{#1}}

\def\1{\bm{1}}

\def\re{{\textnormal{e}}}

\DeclareMathAlphabet{\mathsfit}{\encodingdefault}{\sfdefault}{m}{sl}
\SetMathAlphabet{\mathsfit}{bold}{\encodingdefault}{\sfdefault}{bx}{n}

\usepackage{hyperref}
\hypersetup{hidelinks}
\usepackage{url}
\usepackage{booktabs}
\usepackage{amsmath,amssymb}
\usepackage{amsthm}
\newtheorem{proposition}{Proposition}
\usepackage{graphicx}
\usepackage{float}
\usepackage{tikz}
\usetikzlibrary{positioning,arrows.meta,calc,fit}
\usepackage{xcolor}

\title{RAOA: Alternating-Operator Neural Computation with Programmable Radio Propagation}

\author{Toshiaki Koike-Akino\\Mitsubishi Electric Research Laboratories (MERL), 201 Broadway, Cambridge, MA 02139, USA}
\iclrfinalcopy
\begin{document}
\maketitle

\begin{abstract}
Can programmable radio propagation serve as computational depth rather than only as a communication channel or one-shot analog transform? We introduce the \textbf{Radio Alternating Operator Ansatz (RAOA)}, a recurrent computing architecture that alternates an energy-derived problem update with a mixing update over a persistent latent state. Recomputing the problem field after each mix makes repeated passes compositional even when the same learned controls are reused across depth. We evaluate this idea through exact discrete optimization, constrained programmable-propagation simulation, and pretrained-model adaptation. On discrete objectives, repeated execution can improve solution quality without increasing the learned-control count, and the same formulation handles higher-order interactions directly. A passive phase-only free-space model further shows that the required operators can be approximated by programmable propagation while retaining useful downstream behavior despite realization error. When inserted as a zero-initialized residual adapter, RAOA adapts pretrained language models with WikiText performance close to a matched shallow MLP across three model families, while reasoning-task transfer remains model-dependent. Together, these results connect alternating-operator computation, programmable radio propagation, and neural adaptation within one recurrent framework. The RF realization evidence is simulation-based rather than a hardware demonstration.
\end{abstract}

\section{Introduction}\label{sec:intro}

Wireless systems usually separate \emph{propagation} from \emph{computation}: radio waves carry information, then digital processors transform the received samples. Reconfigurable intelligent surfaces (RISs), programmable metasurfaces, computational metamaterials, and physical neural networks blur this boundary by making the transfer function itself programmable~\citep{silva2014metamaterials,cui2014coding,zangeneh2021analogue,wright2022physical,momeni2025training}. This motivates a broader question than beam steering or one-shot analog transforms: \emph{can a programmable radio path be reused as computational depth?}

We study this question through the \textbf{Radio Alternating Operator Ansatz (RAOA)}. RAOA maintains a latent state and alternates a problem-dependent pass with a mixing pass. The local field is re-evaluated after mixing, so repeated modeled passes compose state-dependent transformations even when the learned controls are tied across depth. Figure~\ref{fig:adapter} shows the intended interpretation: the same programmable path is reconfigured between the two pass types, receiver state is fed back, and input and output projections wrap the recurrent state as a residual Transformer adapter. The experiments instantiate this loop as a differentiable software model, not an RF prototype.

\begin{figure}[t]
\centering
\resizebox{0.96\linewidth}{!}{\begin{tikzpicture}[x=1cm,y=1cm,font=\small,>=Latex]
\tikzset{
  flow/.style={-{Latex[length=2.2mm]},line width=0.8pt,draw=black!60},
  feedback/.style={-{Latex[length=2.2mm]},line width=0.8pt,draw=teal!55!black},
  box/.style={draw=black!45,rounded corners=2pt,fill=gray!4,align=center,inner sep=4pt},
  softblue/.style={draw=blue!35!black,rounded corners=2pt,fill=blue!7,align=center,inner sep=4pt},
  softorange/.style={draw=orange!45!black,rounded corners=2pt,fill=orange!9,align=center,inner sep=4pt},
  softgreen/.style={draw=green!35!black,rounded corners=2pt,fill=green!7,align=center,inner sep=4pt},
  panel/.style={draw=black!20,rounded corners=3pt,fill=black!1,line width=0.6pt}
}

\node[anchor=west,font=\bfseries] at (0,4.20) {(a) Conceptual physical RAOA loop};
\draw[panel] (-0.15,0.08) rectangle (15.45,4.00);

\node[softblue,minimum width=1.95cm] (pctrl) at (3.20,3.34) {problem control\\$\gamma_\ell,\,J,\,b(c)$};
\node[softorange,minimum width=1.70cm] (mctrl) at (10.70,3.34) {mixer control\\$\beta_\ell$};

\node[font=\scriptsize,align=center] at (0.95,3.26) {transmitter\\encodes $z^{(\ell)}$};
\draw[black!60,line width=0.75pt] (0.95,2.30)--(0.95,1.30);
\draw[black!60,line width=0.75pt] (0.95,2.30)--(0.72,2.63);
\draw[black!60,line width=0.75pt] (0.95,2.30)--(1.18,2.63);
\draw[black!60,line width=0.75pt] (0.72,2.63)--(1.18,2.63);

\node[font=\scriptsize,align=center] at (5.55,3.26) {receiver\\samples state};
\draw[black!60,line width=0.75pt] (5.55,2.30)--(5.55,1.30);
\draw[black!60,line width=0.75pt] (5.55,2.30)--(5.32,2.63);
\draw[black!60,line width=0.75pt] (5.55,2.30)--(5.78,2.63);
\draw[black!60,line width=0.75pt] (5.32,2.63)--(5.78,2.63);

\draw[fill=blue!9,draw=blue!45!black,line width=0.7pt] (2.50,1.10) rectangle (3.90,2.40);
\foreach \x in {2.78,3.06,3.34,3.62}{\draw[blue!35] (\x,1.10)--(\x,2.40);}
\foreach \y in {1.36,1.62,1.88,2.14}{\draw[blue!35] (2.50,\y)--(3.90,\y);}
\node[font=\scriptsize,align=center] at (3.20,0.66) {programmable aperture\\problem configuration $\Phi_P$};
\draw[flow] (pctrl.south)--(3.20,2.45);
\draw[flow] (1.28,1.76)--(2.41,1.76);
\draw[flow] (3.99,1.76)--(5.25,1.76);
\foreach \yy in {1.50,1.76,2.02}{\draw[blue!35,line width=0.55pt] (1.53,\yy) .. controls (1.83,\yy+0.18) and (2.08,\yy-0.18) .. (2.33,\yy);}
\foreach \yy in {1.50,1.76,2.02}{\draw[blue!35,line width=0.55pt] (4.07,\yy) .. controls (4.40,\yy+0.18) and (4.72,\yy-0.18) .. (5.05,\yy);}

\node[softgreen,minimum width=1.48cm] (mid) at (7.10,0.96) {$\tilde z^{(\ell)}$\\feedback /\\state};
\draw[feedback] (5.55,1.24) -- (5.55,0.96) -- (mid.west);

\node[font=\scriptsize,align=center] at (8.25,3.26) {transmitter\\re-encodes state};
\draw[black!60,line width=0.75pt] (8.25,2.30)--(8.25,1.30);
\draw[black!60,line width=0.75pt] (8.25,2.30)--(8.02,2.63);
\draw[black!60,line width=0.75pt] (8.25,2.30)--(8.48,2.63);
\draw[black!60,line width=0.75pt] (8.02,2.63)--(8.48,2.63);
\draw[feedback] (mid.east) -- (8.02,0.96);

\node[font=\scriptsize,align=center] at (12.85,3.26) {receiver\\produces $z^{(\ell+1)}$};
\draw[black!60,line width=0.75pt] (12.85,2.30)--(12.85,1.30);
\draw[black!60,line width=0.75pt] (12.85,2.30)--(12.62,2.63);
\draw[black!60,line width=0.75pt] (12.85,2.30)--(13.08,2.63);
\draw[black!60,line width=0.75pt] (12.62,2.63)--(13.08,2.63);

\draw[fill=orange!10,draw=orange!50!black,line width=0.7pt] (10.00,1.10) rectangle (11.40,2.40);
\foreach \x in {10.28,10.56,10.84,11.12}{\draw[orange!42] (\x,1.10)--(\x,2.40);}
\foreach \y in {1.36,1.62,1.88,2.14}{\draw[orange!42] (10.00,\y)--(11.40,\y);}
\node[font=\scriptsize,align=center] at (10.70,0.66) {same RF path, reprogrammed\\mixer configuration $\Phi_M$};
\draw[flow] (mctrl.south)--(10.70,2.45);
\draw[flow] (8.58,1.76)--(9.91,1.76);
\draw[flow] (11.49,1.76)--(12.55,1.76);
\foreach \yy in {1.50,1.76,2.02}{\draw[blue!35,line width=0.55pt] (8.83,\yy) .. controls (9.13,\yy+0.18) and (9.46,\yy-0.18) .. (9.83,\yy);}
\foreach \yy in {1.50,1.76,2.02}{\draw[blue!35,line width=0.55pt] (11.57,\yy) .. controls (11.90,\yy+0.18) and (12.22,\yy-0.18) .. (12.35,\yy);}

\node[box,fill=gray!3,font=\scriptsize,minimum width=1.40cm] (next) at (14.15,1.76) {next pair\\$\ell\!+\!1$};
\draw[feedback] (13.08,1.76)--(next.west);
\coordinate (loopstart) at (13.95,0.38);
\coordinate (loopend) at (6.35,0.38);
\draw[feedback,dashed] (next.south) |- (loopstart) -- (loopend);
\node[font=\scriptsize,anchor=east,text=teal!55!black] at (6.20,0.38) {repeat stage sequence};

\node[anchor=west,font=\bfseries] at (0,-0.30) {(b) Zero-initialized Transformer insertion};
\draw[panel] (-0.15,-3.25) rectangle (15.45,-0.48);

\node[box,minimum width=1.15cm] (h) at (0.90,-1.80) {hidden\\state $h$};
\node[box,minimum width=1.10cm] (ln) at (2.35,-1.80) {pre-\\norm};
\node[softgreen,minimum width=1.32cm] (win) at (4.00,-1.80) {$W_{\rm in}$\\radio state};
\node[softblue,minimum width=3.55cm,minimum height=1.05cm] (loop) at (7.10,-1.80) {conceptual RAOA loop above\\$\mathcal M_{\beta_\ell}\!\circ\!\mathcal P_{\gamma_\ell}$, repeated $L$ times};
\node[softgreen,minimum width=1.22cm] (wout) at (10.25,-1.80) {$W_{\rm out}$\\$\Delta h$};
\node[box,minimum width=1.40cm] (gate) at (12.10,-1.80) {residual scale\\$a$, $a_0=0$};
\node[circle,draw=black!45,fill=gray!3,minimum size=0.54cm,inner sep=0pt] (plus) at (13.60,-1.80) {$+$};
\node[box,minimum width=0.95cm] (hout) at (14.60,-1.80) {$h'$};

\draw[flow] (h.east) -- (ln.west);
\draw[flow] (ln.east) -- (win.west);
\draw[flow] (win.east) -- (loop.west);
\draw[flow] (loop.east) -- (wout.west);
\draw[flow] (wout.east) -- (gate.west);
\draw[flow] (gate.east) -- (plus.west);
\draw[flow] (plus.east) -- (hout.west);
\draw[flow] (h.south) |- (13.60,-2.78) -- (plus.south);

\end{tikzpicture}}
\caption{\textbf{Conceptual RAOA realization and Transformer insertion.} (a) One stage alternates problem and mixer configurations of the same programmable path; receiver state is fed back, so $L$ stages correspond to $2L$ modeled RF passes. (b) A pretrained Transformer projects hidden state into the RAOA state, returns it to model width, and adds it through a zero-initialized residual scale, making the insertion exactly behavior-preserving at initialization. The physical loop is conceptual; all foundation-model experiments execute the differentiable operator in software.}
\label{fig:adapter}
\end{figure}
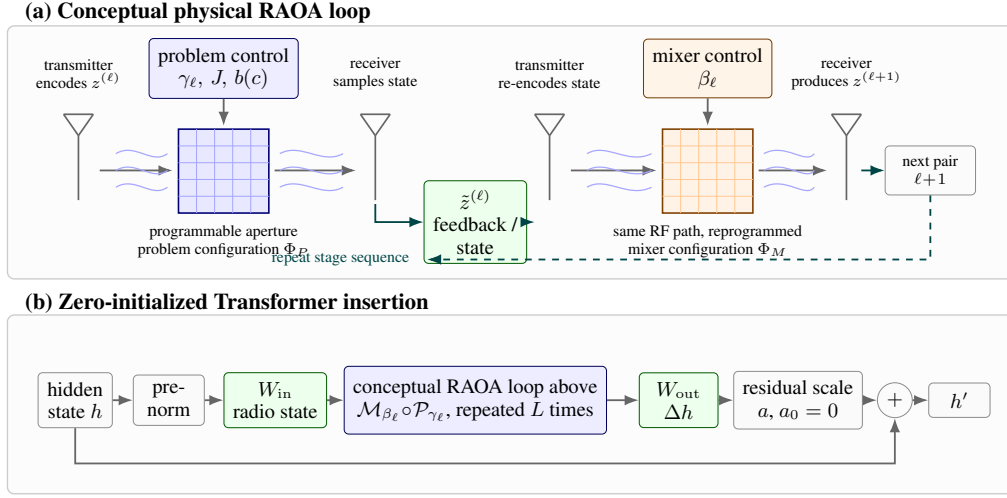

We use two experimental regimes for deliberately different claims. First, RAOA acts as a classical alternating-operator optimizer on exact QUBO/HUBO objectives: known optima let us vary modeled pass depth without changing learned-control count, and independent random instances plus cross-depth reuse of the same two-control program test whether repeated execution itself adds capability. A constrained passive phase-only free-space model then asks whether useful behavior survives explicit propagation. Pretrained language models test a separate \emph{transfer} question: can a zero-initialized RAOA stage adapt realistic frozen checkpoints while retaining digital-adapter quality? Large-model gains are not depth evidence: the primary 4B comparison fixes depth.

RAOA does not claim wave-domain inference, nonlinear computation from programmable scattering, iterative wave computation, or parameter-independent physical depth as new in isolation. Prior systems already demonstrate recurrent and nonlinear wave computation~\citep{hughes2019wave,wanjura2024nonlinear,xia2024recurrent,yildirim2024nonlinear}, closed-loop programmable RF computation~\citep{tzarouchis2025programmable}, experimentally trained metasurface computational environments with in situ electromagnetic feedback~\citep{song2026computationalspace}, wireless/metasurface neural layers~\citep{stylianopoulos2025ota,huang2025simsemcom,an2026wave}, and shared-weight programmable-metasurface depth~\citep{hammami2026expressivity}. The contribution here is the specific \emph{temporal problem/mixer alternation} over a recurrent state, together with controlled tests on exact polynomial objectives and zero-initialized insertion into pretrained Transformers.

Our contributions are:
\begin{itemize}
    \item \textbf{A task-structured recurrent operator.} RAOA alternates energy-derived problem and mixing passes over a persistent latent state. A common local-field formulation covers learned quadratic/context operators and exact QUBO/HUBO objectives, while tied controls isolate pass depth from parameter growth.
    \item \textbf{Alternating-operator depth gains across independent instances.} Across independently drawn $n=8,12$ instances, each instance fixes its own objective coefficients and learns only two shared scalar controls $(\gamma,\beta)$. Simply extending that instance-specific two-control program from two to 4, 8, and 16 passes increases mean success for all four tested objective families. Matched-depth optimization strongly benefits 3-SAT, Max-Cut, and hypergraph Max-Cut but is inconsistent on partition, revealing optimization stability rather than parameter count as the main caveat. Native higher-order fields also avoid the fixed-budget ancilla penalty of the tested exact quadratizations.
    \item \textbf{A constrained propagation bridge and pretrained-model transfer.} Passive phase-only free-space compilation improves systematically with panel count and retains meaningful Max-Cut behavior despite nonzero operator error. Separately, single-stage RAOA adapters recover 93.9--99.2\% of the shallow MLP WikiText NLL improvement across three Transformer families under matched validation budgets, while GSM8K transfer remains model-dependent. All radio-propagation results are simulated.
\end{itemize}

\section{Related Work}\label{sec:related}

\paragraph{Wave and metasurface computation.}
Metamaterial analog computing maps operators such as differentiation and integration into engineered propagation~\citep{silva2014metamaterials,zangeneh2021analogue}. Coding and programmable metasurfaces add electronic reconfiguration~\citep{cui2014coding}, while RIS work traditionally optimizes communication and sensing channels~\citep{wu2019joint,wu2020smart,liu2021ris}. More recent systems move closer to computation itself: \citet{tzarouchis2025programmable} demonstrate a programmable RF wave machine for open- and closed-loop linear algebra and iterative algorithms, while \citet{song2026computationalspace} experimentally train a four-metasurface computational environment from in situ electromagnetic feedback without a differentiable physical model. RAOA therefore does not claim iterative wave computation or metasurface-based physical learning in general; it studies a learned recurrent \emph{alternation} between problem and mixer passes, with explicit exact-objective semantics and pass-depth controls.

\paragraph{Physical neural networks, nonlinearity, and depth.}
Diffractive and coherent photonic networks and broader physical neural networks implement learned transformations directly in propagation~\citep{shen2017nanophotonic,lin2018d2nn,wright2022physical,li2024mechanical,momeni2025training}. Nonlinear computation can also emerge from nominally linear scattering through structural input encoding, repeated modulation, multiple scattering, or recurrence~\citep{hughes2019wave,wanjura2024nonlinear,xia2024recurrent,yildirim2024nonlinear}. Most directly, \citet{hammami2026expressivity} separate encoding nonlinearity, structural nonlinearity, and programmable-metasurface WPNN depth and show that shared-weight depth can improve expressivity. Accordingly, our novelty is not ``more wave depth without more trainable weights'' by itself. Their study varies structural encoding, coupling, and layer depth; RAOA instead studies recurrent feedback over one state, problem/mixer semantics, exact QUBO/HUBO fields, cross-depth control reuse, and zero-initialized Transformer insertion.

\paragraph{Stacked intelligent metasurfaces and over-the-air neural inference.}
Stacked intelligent metasurfaces (SIMs) realize cascaded trainable electromagnetic layers for semantic inference and wave-domain processing~\citep{huang2025simsemcom,liu2025sim,an2026wave,stylianopoulos2026minn}. Metasurfaces-Integrated Neural Networks further optimize RIS/SIM-enabled wireless propagation as DNN layers for end-to-end edge inference~\citep{stylianopoulos2025ota}. These results establish wave-domain neural inference in wireless systems. The RAOA experiments study a complementary axis: repeated temporal feedback reuse of one latent state through alternating pass types, including exact optimization and foundation-model adaptation.

\paragraph{Structured iteration and parameter-efficient adaptation.}
RAOA relates to algorithm unrolling~\citep{gregor2010lista,monga2021unrolling}, neural operators~\citep{kovachki2023neuraloperator,li2021fno}, and alternating-operator constructions such as QAOA~\citep{farhi2014qaoa}; the QAOA link is architectural only, since RAOA is classical, real-valued, and trained by backpropagation. For foundation-model adaptation, LoRA~\citep{hu2022lora} is the low-rank baseline and a shallow MLP the unconstrained nonlinear reference. We ask how much digital-adapter quality a physically structured alternating operator retains, not whether it universally dominates (Appendix~\ref{app:related}).

\section{Radio Alternating Operator Ansatz}\label{sec:method}

\subsection{From programmable propagation to iterative computation}
A fixed programmed path $y=H\Phi x$ is linear in the transmitted state. Signal-conditioned programming $\Phi(x)$ can make the end-to-end map nonlinear even when each instantaneous propagation step is linear~\citep{wanjura2024nonlinear,yildirim2024nonlinear}; receiver feedback then adds recurrent state. RAOA gives this recurrence a specific structure: each stage applies a \emph{problem pass} followed by a \emph{mixer pass}, and the pair can be repeated while reusing the same controls.

\begin{figure*}[t]
\centering
\resizebox{0.98\textwidth}{!}{\begin{tikzpicture}[font=\scriptsize,>=Latex,node distance=0.46cm]
\tikzset{
  stage/.style={draw=black!35,rounded corners=3pt,align=center,minimum width=3.65cm,minimum height=1.58cm,inner sep=5pt,line width=0.7pt},
  arr/.style={-{Latex[length=2mm]},line width=0.8pt,draw=black!55}
}
\node[stage,fill=blue!6] (static) {\textbf{Static surface}\\$y=H\Phi x$\\\footnotesize fixed linear map\\\footnotesize cubic MSE $2.37\!\times\!10^{-2}$};
\node[stage,fill=green!7,right=of static] (cond) {\textbf{Signal-conditioned}\\$y=H\Phi(x)x$\\\footnotesize input-dependent map\\\footnotesize cubic MSE $\approx0$ ideal; $2.4\!\times\!10^{-3}$ severe};
\node[stage,fill=orange!8,right=of cond] (fb) {\textbf{Feedback state}\\$z_{t+1}=F(z_t,x_t)$\\\footnotesize memory / recurrence\\\footnotesize delayed product: $49.3\%$ linear $\rightarrow100\%$ nonlinear};
\node[stage,fill=violet!7,right=of fb] (raoa) {\textbf{RAOA alternation}\\$\mathcal M_{\beta}\!\circ\!\mathcal P_{\gamma}$, repeated\\\footnotesize re-evaluate field after mixing\\\footnotesize $L$ stages $=2L$ modeled RF passes};
\draw[arr] (static) -- node[above=3pt,font=\scriptsize,fill=white,inner sep=1pt]{condition} (cond);
\draw[arr] (cond) -- node[above=3pt,font=\scriptsize,fill=white,inner sep=1pt]{feed back} (fb);
\draw[arr] (fb) -- node[above=3pt,font=\scriptsize,fill=white,inner sep=1pt]{alternate} (raoa);
\end{tikzpicture}}
\caption{\textbf{Operator ladder motivating RAOA.} A fixed programmed path is linear in its transmitted state; signal-conditioned programming can create nonlinear input-output maps; feedback adds memory; RAOA repeatedly alternates problem and mixer roles over that recurrent state. The quantitative annotations are controlled primitive demonstrations, not application benchmarks or hardware measurements. ``Severe'' denotes the preregistered synthetic impairment profile.}
\label{fig:ladder}
\end{figure*}
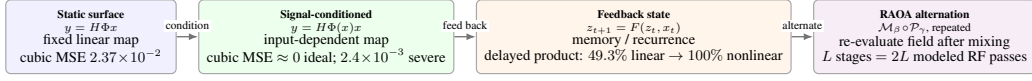

Figure~\ref{fig:ladder} makes the progression explicit. Our question is whether composing state-dependent problem updates with recurrent mixing makes repeated programmable passes useful as computational depth, not whether any individual ingredient is new.

\subsection{RAOA update}
Let $z^{(\ell)}\in\mathbb{R}^{m}$ be the recurrent state, $c$ a context, and $E_\psi(z;c)$ a differentiable problem energy. One stage (two modeled RF passes) is
\begin{align}
 u^{(\ell)}&=\nabla_zE_\psi(z^{(\ell)};c),\quad s^{(\ell)}=\operatorname{mean}_i|u_i^{(\ell)}|+\epsilon,\nonumber\\[-1mm]
 \tilde z^{(\ell)}&=\tanh\!\left(z^{(\ell)}-\sin\gamma_\ell\,u^{(\ell)}/s^{(\ell)}\right),\quad
 z^{(\ell+1)}=\tanh\!\left(\cos\beta_\ell\,\tilde z^{(\ell)}+\sin\beta_\ell\,M\tilde z^{(\ell)}\right).
 \label{eq:raoa_compact}
\end{align}
Here $M$ is a fixed Hadamard-derived orthogonal mixer. Crucially, the next problem pass recomputes $\nabla E_\psi$ on the \emph{new} state. Tying $\gamma_\ell\equiv\gamma$ and $\beta_\ell\equiv\beta$ increases execution depth without adding to the two learned scalar controls. In exact optimization, the coefficients of $E_{\rm exact}$ are fixed by the instance; Sec.~\ref{sec:results_controls} tests whether the added composition helps.

The same update covers both study regimes:
\begin{equation}
 E_{\rm FM}(z;c)=\tfrac12 z^\top Jz+(W_{\rm in}c)^\top z,
 \qquad
 E_{\rm exact}(z)=\sum_{S}w_S\!\prod_{i\in S}z_i.
 \label{eq:energy_specializations}
\end{equation}
The first gives the learned quadratic field used in Transformer adapters; the second directly represents a known QUBO or HUBO objective. In the discrete setting, RAOA is a \emph{classical alternating-operator optimizer}: a problem-derived update and a mixer update repeat, architecturally paralleling QAOA and the broader alternating-operator ansatz~\citep{farhi2014qaoa,hadfield2019alternating}. The analogy stops there---RAOA uses a real recurrent state, nonlinear classical updates, and ordinary backpropagation, with no quantum-equivalence or quantum-advantage claim. Appendix~\ref{app:method} gives the full local-field form, normalization invariance, parameterization, and implementation details. Native higher-order fields are phenomenological software operators, not evidence of RF many-body realization.

\paragraph{Why tied depth can add computation.}
A compact surrogate makes the role of repeated execution explicit. In the unsaturated quadratic case, we freeze the positive scalar normalization within the trajectory and define $B_\beta=\cos\beta I+\sin\beta M$, $\eta=\sin\gamma/s$, $A=B_\beta(I-\eta J)$, and $c=-\eta B_\beta b$. With tied controls,
\begin{equation}
 \bar z^{(L)}=A^L\bar z^{(0)}+\sum_{k=0}^{L-1}A^k c.
 \label{eq:raoa_affine_power}
\end{equation}
Thus additional passes expose higher powers of the same stage operator without adding learned controls, giving a block-Krylov-style unrolled computation~\citep{saad2003iterative,monga2021unrolling}. Appendix~\ref{app:theory} proves this under the stated surrogate assumptions and adds three results: dynamic polynomial-degree growth versus a frozen field, local descent of the problem substep under the standard smoothness condition, and a finite-depth approximation-error bound for the \texttt{tanh}-bounded recurrence. These explain how depth \emph{can} help, not that it must.

\subsection{Conceptual radio realization}
Figure~\ref{fig:adapter} interprets one stage physically: a transmitter encodes the state, a programmable aperture executes the problem pass, receiver output is fed back, and the path is reprogrammed for the mixer. Repeating the stage yields $2L$ modeled RF passes. The matrices in Eq.~\ref{eq:raoa_compact} are effective receiver-plane operators, not literal surface coefficients; Sec.~\ref{sec:results_structure} tests pairwise problem/mixer operators under explicit passive phase-only propagation.

\subsection{Zero-initialized Transformer insertion}
For hidden state $h$, RAOA uses $c=\operatorname{LN}(h)$ and inserts a residual adapter
\begin{equation}
 z^{(0)}=\tanh(W_{\rm in}c),\qquad
 \Delta h=W_{\rm out}z^{(L)},\qquad
 f_{\rm adapt}(h)=f_{\rm base}(h)+a\,\Delta h,\quad a_0=0.
 \label{eq:adapter_compact}
\end{equation}
The zero-initialized residual scale makes insertion exactly behavior-preserving at initialization; adaptation can subsequently change model behavior. The primary 4B experiments use $L=1$; tied controls are used in the exact depth studies to isolate repeated execution from parameter growth. Exact parameter counts and bias conventions are deferred to Appendix~\ref{app:method}.

\section{Experimental Design}\label{sec:experiments}

We test nonlinear and recurrent ingredients, fixed-control depth, zero-initialized pretrained-model transfer, and survival under constrained propagation; all reported quantitative values are generated from frozen JSON/CSV artifacts.

\paragraph{Exact discrete optimization.}
We use RAOA as a classical alternating-operator solver on four objective families spanning pairwise QUBOs and native higher-order HUBOs:
\begin{center}
\small
\setlength{\tabcolsep}{4.2pt}
\begin{tabular}{@{}llll@{}}
\toprule
Objective & Structure & Order & Role in study\\
\midrule
Max-Cut & weighted graph cut & 2 & canonical pairwise objective\\
Partition & squared imbalance & 2 & optimization counterexample\\
3-SAT & unsatisfied-clause penalty & 3 & native higher-order objective\\
Hypergraph Max-Cut & monochromatic-edge penalty & 3--4 & native higher-order objective\\
\bottomrule
\end{tabular}
\end{center}
All optima are enumerated. The sweep uses 12 independent instances per family at $n\in\{8,12\}$, $L\in\{1,2,4,8\}$, and five seeds. One $(\gamma,\beta)$ pair is tied across stages, so every depth has two learned controls; each pair is also executed at every other depth without reoptimization. Cross-depth details are in Appendix~\ref{app:random_depth}; exact encodings and controls in Appendix~\ref{app:discrete_optimization}.

\paragraph{Constrained propagation.}
For pairwise Max-Cut and partition at $n=8$, we compile problem and mixer operators through 1, 2, or 4 coherent passive phase-only free-space panels (128 elements each), with no trainable amplitude control, no dense digital precoder or combiner, and no direct path. We measure operator error and downstream success both with ideal controls reused and after reoptimizing only the same two controls. Appendix~\ref{app:phase_only} gives geometry and compiler details.

\paragraph{Foundation models.}
We evaluate Qwen3-4B, Gemma-3-4B-it, and Phi-4-mini-instruct~\citep{yang2025qwen3,gemma2025technical,microsoft2025phi4mini}. Primary RAOA adapters use state width $m=64$, $L=1$, and all text MLP blocks, testing architectural transfer rather than large-model depth scaling. WikiText~\citep{merity2016wikitext} uses disjoint training and validation data followed by a three-seed full-test confirmation. GSM8K~\citep{cobbe2021gsm8k} selects schedules on a training-split holdout before all 1,319 official test problems. After observing material 384-token cap-hit rates, a preregistered audit re-evaluated the frozen Qwen and Gemma schedules at 768 tokens without retuning and stored per-example predictions for a paired bootstrap analysis.

\paragraph{Baselines and controls.}
We compare frozen, LoRA, RAOA, and a shallow bottleneck MLP under equal family-specific validation-search budgets and closely matched trainable parameter counts. Appendix~\ref{app:experiments} gives placement, the Qwen depth diagnostic, schedules, generation diagnostics, and parameter counts; Sec.~\ref{app:hardware} gives the synthetic perturbation study.

\section{Results}\label{sec:results}

\subsection{Nonlinear and recurrent computational primitives}\label{sec:results_hierarchy}
Figure~\ref{fig:ladder} combines the operator progression with two controlled primitive demonstrations. A fixed scalar linear surface cannot represent $x\mapsto x^3$: its best-fit MSE is $2.37\times10^{-2}$, while signal-conditioned control reaches numerically zero error in the ideal model and $2.4\times10^{-3}$ under the severe synthetic profile. Linear feedback is likewise near chance on delayed multiplication ($49.3\%$), while nonlinear feedback reaches $100\%$ ideally. These minimal tests isolate the input-dependent and recurrent nonlinear primitives that RAOA composes rather than benchmark application performance.

\subsection{Discrete optimization with alternating radio passes}\label{sec:results_controls}
Here RAOA is evaluated as an alternating-operator optimizer, not a functionality check. Figure~\ref{fig:discrete_main} isolates execution depth from parameter growth and depth-specific reoptimization. For each random instance, we fix its objective coefficients, optimize two tied scalar controls $(\gamma,\beta)$ at $L=1$ (two passes), then reuse them unchanged at 4, 8, and 16 passes. This is replication across independent instances, not cross-instance control transfer. Mean success rises monotonically for every family-size group. At $n=12$, 3-SAT increases from $0.367$ to $0.980$, hypergraph Max-Cut from $0.441$ to $0.984$, Max-Cut from $0.0158$ to $0.175$, and partition from $0.0304$ to $0.0594$; the corresponding $n=8$ endpoints are $0.543\rightarrow0.997$, $0.586\rightarrow0.995$, $0.0785\rightarrow0.390$, and $0.0356\rightarrow0.0733$. Thus useful recurrent computation can arise from simply running the same learned two-control program for more passes.

\begin{figure}[H]
\centering
\includegraphics[width=0.97\linewidth]{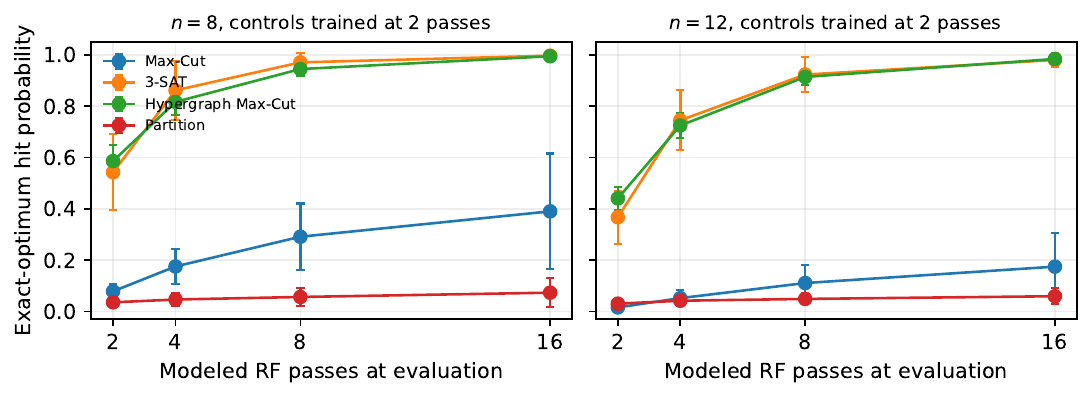}
\caption{\textbf{Cross-depth reuse isolates recurrent execution depth.} Twelve independent instances per family and size; error bars are standard deviations across instances after averaging five optimizer seeds within each instance. The same $(\gamma,\beta)$ controls are optimized only at two passes and then reused unchanged at deeper execution depths. Mean exact-optimum hit probability increases monotonically in all eight family-size groups.}
\label{fig:discrete_main}
\end{figure}

Matched-depth optimization gives a complementary view. All 12 3-SAT instances improve at every deeper setting for both $n=8$ and $n=12$. At 16 passes, hypergraph Max-Cut improves over the two-pass baseline on 11/12 and 12/12 instances, while Max-Cut improves on 10/12 and 12/12. Partition is the optimization counterexample: at $n=8$, only 2/12 instances improve and 10/12 worsen at 16 passes; at $n=12$, the split is 5/12 versus 7/12. Because shallow-trained partition controls still improve on average when merely executed deeper, the earlier partition failure is better localized to matched-depth optimization instability rather than absence of useful recurrent execution depth. An untied-control sweep likewise shows no systematic advantage from growing the control count: the four-problem mean is $0.570$ tied versus $0.569$ untied at eight passes (2 versus 8 controls) and $0.566$ versus $0.532$ at sixteen passes (2 versus 16 controls; Appendix~\ref{app:random_depth}).

A preregistered mechanism ablation further localizes recurrence. At eight passes, freezing the initial local field drops success from $0.985$ to $0.483$ on 3-SAT, $0.850$ to $0.040$ on hypergraph Max-Cut, and $0.430$ to $0.060$ on Max-Cut. Mixer ablations are objective-dependent: nearly neutral on 3-SAT and hypergraph Max-Cut, about $0.11$ worse on Max-Cut, and better on the low partition baseline. Thus state-dependent field re-evaluation is the consistent mechanism; the Hadamard mixer is not universally necessary or superior (Appendix~\ref{app:mechanism_ablation}).

\paragraph{A structural interpretation of recurrent depth.}
Appendix~\ref{app:theory} formalizes why field re-evaluation can matter beyond repeated application of a fixed map. For a degree-$d$ polynomial energy, the unsaturated fixed-normalization surrogate permits
\begin{equation}
 \deg\!\left(G^{\circ L}\right)\le(d-1)^L,
 \qquad
 \deg\!\left(z^{(L)}_{\rm frozen}\right)\le d-1,
 \label{eq:main_degree_growth}
\end{equation}
where the second bound freezes the initial local field for all passes. The full normalized, \texttt{tanh}-saturated RAOA is not polynomial, so Eq.~\ref{eq:main_degree_growth} is a structural surrogate rather than an exact description of the implemented map. Together with the ablation above, it identifies dynamic field re-evaluation as a mechanism by which higher-order objectives can expose additional compositional order with depth, while also explaining why merely repeating a frozen field cannot do so.

RAOA's direct higher-order problem layer is also useful for alternating-operator optimization beyond QUBOs. Under the same eight-pass, two-control budget, exact quadratization introduces 8--9 ancillas and sharply lowers optimum-hit probability:
\begin{table}[H]
\centering
\small
\setlength{\tabcolsep}{4.5pt}
\caption{\textbf{Native higher-order execution versus exact quadratization at eight modeled passes.} Values are mean $\pm$ seed standard deviation over five seeds.}
\label{tab:native_main}
\begin{tabular}{@{}lrrr@{}}
\toprule
Problem / representation & Executed vars. & Ancillas & Success\\
\midrule
3-SAT, native HUBO & 8 & 0 & $0.989\pm0.003$\\
3-SAT, quadratized QUBO & 17 & 9 & $0.161\pm0.023$\\
Hypergraph Max-Cut, native HUBO & 8 & 0 & $0.830\pm0.008$\\
Hypergraph Max-Cut, quadratized QUBO & 16 & 8 & $0.119\pm0.011$\\
\bottomrule
\end{tabular}
\end{table}
This fixed-budget result does not establish universal HUBO superiority; it shows the ancilla burden of the tested exact reductions when higher-order local fields are available natively. Appendix~\ref{app:discrete_optimization} reports complete pass sweeps and encodings. Coefficient-RMS normalization also leaves 15 of 16 Max-Cut and partition depth comparisons exactly unchanged, so simple global objective scale does not explain the optimization instability.

\subsection{Foundation-model transfer}\label{sec:results_foundation}
Table~\ref{tab:foundation_main} summarizes the full-test foundation-model evidence. On WikiText, RAOA recovers 96.0\%, 99.2\%, and 93.9\% of the shallow MLP NLL improvement over frozen Qwen, Gemma, and Phi. It remains close to the generic nonlinear reference without being the universal winner: LoRA leads on Qwen, MLP on Gemma and Phi.

Trainable budgets are tightly matched. Across Qwen, Gemma, and Phi, RAOA uses 12.04/11.37/12.81M parameters versus 11.98/11.32/12.78M for LoRA and 12.08/11.40/12.88M for MLP, staying within 0.6\% of each digital reference. Full-test WikiText evaluations include 296.5k, 289.9k, and 284.7k scored next-token positions, respectively (Appendix~\ref{app:experiments}).

A validation-selected, parameter-matched Qwen depth-2 digital recurrent bottleneck (12,038,436 parameters versus RAOA's 12,038,544) reaches full-test NLL $2.6727\pm0.0029$, between MLP ($2.6596\pm0.0023$) and RAOA ($2.6851\pm0.0040$). Thus WikiText establishes effective RAOA transfer, not a unique advantage over generic recurrence (Appendix~\ref{app:recurrent_baseline}).

\begin{table}[H]
\centering
\small
\setlength{\tabcolsep}{3.2pt}
\caption{\textbf{Foundation-model results.} WikiText NLL (lower is better) and GSM8K EM in percent (higher is better). Qwen and Gemma use the frozen-schedule 768-token sensitivity audit; Phi retains the canonical 384-token evaluation. Trainable entries are mean $\pm$ seed standard deviation; bold marks the best mean, not significance.}
\label{tab:foundation_main}
\begin{tabular}{@{}llrrrr@{}}
\toprule
Model & Metric & Frozen & LoRA & RAOA ($L{=}1$) & MLP\\
\midrule
Qwen3-4B & WikiText NLL & 3.3007 & $\mathbf{2.6426}\!\pm\!.0016$ & $2.6851\!\pm\!.0040$ & $2.6596\!\pm\!.0023$\\
           & GSM8K EM (768) & 87.95 & $81.65\!\pm\!1.96$ & $87.95\!\pm\!.49$ & $\mathbf{88.05}\!\pm\!.25$\\
Gemma-3-4B-it & WikiText NLL & 5.9833 & $2.9046\!\pm\!.0051$ & $2.9009\!\pm\!.0055$ & $\mathbf{2.8763}\!\pm\!.0061$\\
           & GSM8K EM (768) & 81.12 & $79.83\!\pm\!.98$ & $\mathbf{81.43}\!\pm\!.27$ & $81.25\!\pm\!.29$\\
Phi-4-mini & WikiText NLL & 3.1010 & $2.6232\!\pm\!.0006$ & $2.6503\!\pm\!.0058$ & $\mathbf{2.6210}\!\pm\!.0023$\\
           & GSM8K EM (384) & 79.23 & $83.04\!\pm\!.23$ & $79.45\!\pm\!.39$ & $\mathbf{83.73}\!\pm\!.53$\\
\bottomrule
\end{tabular}
\end{table}

The canonical Qwen and Gemma cap-hit rates were 10.7--16.3\%; the preregistered 768-token audit raises EM by 5.1--8.0 points while reducing cap hits to 0.4--4.4\%. Paired RAOA$-$Frozen differences are $0.00$ points on Qwen (95\% CI $[-0.78,0.78]$) and $+0.30$ on Gemma ($[-0.53,1.16]$); RAOA$-$MLP is $-0.10$ ($[-0.66,0.43]$) and $+0.18$ ($[-0.45,0.81]$). Thus the audit does not distinguish RAOA, MLP, and frozen. Qwen LoRA remains $6.29$ points below frozen ($[-8.11,-4.50]$) despite lowering teacher-forced NLL from $2.043$ to $0.593$. Intervals condition on the completed seed ensemble. Phi was not rerun because adapted 384-token cap-hit rates are mostly below 2\%. Qwen validation NLL remains flat across $L=1,2,4$; Appendix~\ref{app:experiments} reports the canonical table and audit.

\begin{figure*}[t]
\centering
\begin{minipage}[t]{0.49\textwidth}
\centering
\includegraphics[width=\linewidth]{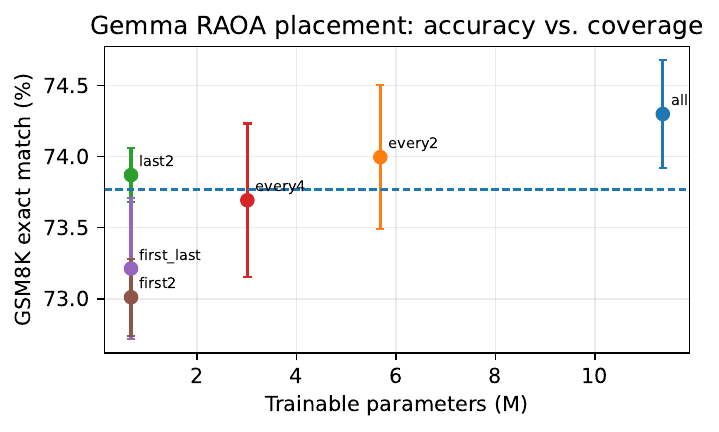}
\end{minipage}\hfill
\begin{minipage}[t]{0.49\textwidth}
\centering
\includegraphics[width=\linewidth]{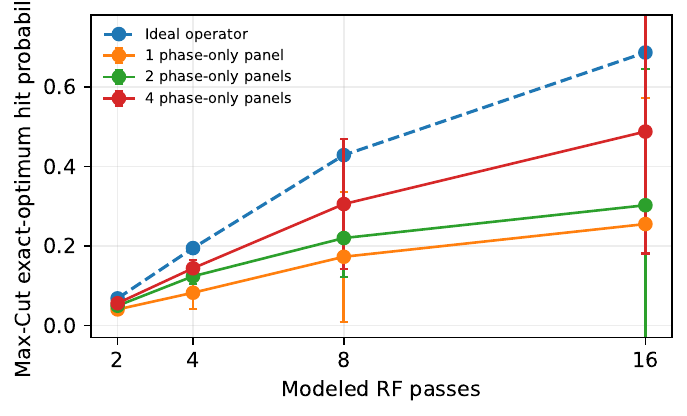}
\end{minipage}
\caption{\textbf{Structural and realizability controls.} \textbf{Left:} full-test Gemma GSM8K accuracy versus RAOA trainable parameters; every-second-layer placement nearly matches dense placement with half the adapters. \textbf{Right:} Max-Cut success after compiling problem/mixer operators into 1, 2, or 4 passive 128-element phase-only panels (coherent parallel single-bounce paths) and reoptimizing only the same two scalar controls. More panel resources reduce the gap to the ideal effective operators, although substantial variability across geometry seeds remains at large depth.}
\label{fig:structure_robustness}
\end{figure*}

\subsection{Structural efficiency and constrained propagation}\label{sec:results_structure}
Figure~\ref{fig:structure_robustness} summarizes two controls with direct system implications. The full Gemma placement sweep rejects the pilot-only early-layer pattern: dense placement reaches $74.30\%$, while every-second-layer placement reaches $74.00\%$ with 17 rather than 34 adapted blocks (5.68M versus 11.37M parameters). Extreme two-layer sparsity returns toward frozen performance, supporting a coverage--efficiency trade-off rather than special early-layer sensitivity.

The right panel tests constrained physical realizability for pairwise objectives. We compile effective problem/mixer operators into passive unit-amplitude free-space panels with no direct path or trainable dense digital maps; the panels contribute as coherent, independent single-bounce paths. Table~\ref{tab:phase_operator_main} shows error falling with 1, 2, and 4 panels. Four panels use 512 programmable phases for a 9-mode target, so this tests feasibility, not efficiency.

\begin{table}[H]
\centering
\small
\setlength{\tabcolsep}{4.1pt}
\caption{\textbf{Phase-only compilation.} Normalized Frobenius error for problem ($P$) and mixer ($M$) operators (mean $\pm$ standard deviation over three geometries); 128 phases per panel and a shared mixer target.}
\label{tab:phase_operator_main}
\begin{tabular}{rrrrr}
\toprule
Panels & Phases & Max-Cut $P$ & Partition $P$ & Mixer $M$\\
\midrule
1 & 128 & $0.687\!\pm\!.016$ & $0.711\!\pm\!.034$ & $0.758\!\pm\!.048$\\
2 & 256 & $0.519\!\pm\!.021$ & $0.535\!\pm\!.059$ & $0.573\!\pm\!.055$\\
4 & 512 & $0.357\!\pm\!.092$ & $0.348\!\pm\!.142$ & $0.406\!\pm\!.092$\\
\bottomrule
\end{tabular}
\end{table}

For Max-Cut, after four-panel compilation, reoptimizing only $(\gamma,\beta)$ gives success probabilities $0.305$ and $0.487$ versus ideal values $0.428$ and $0.686$ at 8 and 16 passes (about 71\% retention); reusing ideal controls gives $0.301$ and $0.410$. Partition's matched-depth ideal baseline is non-monotonic, so we omit a retention ratio. This is constrained simulation, not RF measurement, and does not realize the 64-dimensional foundation adapters or native higher-order fields. Appendix~\ref{app:ris_model} specifies the channel, geometry, and physical assumptions; Appendix~\ref{app:phase_only} gives full operator and downstream details.\par\pagebreak[3]\noindent More generally, Appendix~\ref{app:theory} gives a finite-depth realization bound: if a compiled stage differs uniformly from the ideal by at most $\delta$ on the bounded RAOA state cube and the ideal stage is $K$-Lipschitz, then
\begin{equation}
 \|\widehat z^{(L)}-z^{(L)}\|\le\delta\sum_{k=0}^{L-1}K^k.
 \label{eq:main_compiled_error}
\end{equation}
This is not a contraction claim; it motivates evaluating both pass-level operator error and downstream computation as depth grows, rather than treating matrix approximation alone as a hardware proxy.

Separate synthetic angle/state perturbations remain nearly invariant through the preregistered severe profile: ideal-trained Qwen WikiText NLL changes by only $+0.00051$ ($\approx0.019\%$), while domain-randomized training is slightly worse throughout. We retain this as a software robustness result under modeled perturbations in Appendix~\ref{app:hardware}, not as RF measurement evidence.

\section{Discussion}\label{sec:discussion}

\paragraph{What the evidence establishes.}
Across independent exact instances, deeper execution with the same two learned controls often improves success, while partition separates useful execution depth from successful depth-specific optimization. State-dependent field re-evaluation is the consistent mechanism; mixer benefits are objective-dependent. Passive phase-only compilation provides an imperfect propagation bridge. Together these results support recurrent execution depth as a computational resource on the tested exact objectives, while the single-stage foundation-model study establishes neural transfer rather than language-model depth scaling. Appendix~\ref{app:theory} gives the structural analysis.

\paragraph{Scope.}
The evidence remains operator- and simulation-level: higher-order realization, calibrated hardware, and measured latency and energy remain open. The primary 4B study uses $L=1$; Qwen depth scaling is flat, and a matched depth-2 recurrent control reaches $2.6727\pm0.0029$ NLL, so transfer is not unique to RAOA recurrence. The 768-token GSM8K audit likewise fails to distinguish RAOA, MLP, and frozen under the paired bootstrap. Appendix~\ref{app:limitations} details limitations and negative results.

\section{Conclusion}\label{sec:conclusion}

RAOA casts programmable propagation as a classical alternating-operator recurrence over a persistent state. Across independent exact instances, performance can improve when the same two learned controls are executed for more modeled passes, while partition shows that optimization need not exploit available depth. Passive phase-only compilation provides a constrained pairwise propagation bridge, and a zero-initialized single-stage interface transfers to three pretrained Transformer families while recovering most of a shallow MLP's WikiText NLL gain. The supported claim is focused: recurrent pass depth is a computational resource on the tested objectives, with a constrained-propagation path and transferable neural interface. Measured closed-loop hardware, higher-order physical realization, and foundation-model depth scaling remain open.

\subsection*{AI use statement}
Generative AI (OpenAI ChatGPT) was used to assist with conceptual development and critique, research-software development and review, literature organization, manuscript drafting and editing, and scientific illustration. Figure~\ref{fig:ai_system_overview} is an AI-generated conceptual image; it is included only for intuition and is not experimental evidence. All quantitative plots were generated programmatically from experimental results. All experiments were conducted by the authors, primarily on an NVIDIA RTX~4090 GPU. The authors independently executed and inspected the reported experiments, checked the quantitative claims against the underlying results, and take responsibility for the final content of the submission.

\bibliography{references}
\bibliographystyle{iclr2027_conference}

\appendix
\section{Detailed Related Work}\label{app:related}

\subsection{Computational metamaterials and programmable electromagnetic environments}
Wave-based analog computing predates modern physical neural networks. Metamaterial implementations can encode spatial transfer functions directly in wave propagation, enabling operations such as differentiation and integration~\citep{silva2014metamaterials}. Reviews of computational metamaterials emphasize mapping a desired operator into a compact wave interaction~\citep{zangeneh2021analogue}. Coding and programmable metamaterials add a digital control plane to the medium~\citep{cui2014coding}. \citet{tzarouchis2025programmable} go further with a programmable RF wave machine that supports open-loop matrix operations, closed-loop matrix inversion, root finding, inverse design, and iterative optimization. This is an important precedent: repeated and closed-loop wave computation itself is not unique to RAOA. \citet{song2026computationalspace} raise the hardware bar further by experimentally surrounding a free-space region with four computational metasurfaces and optimizing their discrete phase states from directly measured electromagnetic feedback using a model-free forward-only procedure. Their contribution is in situ physical learning and environmental computation; RAOA instead contributes an explicit recurrent problem/mixer program, controlled tied-depth optimization, and a transferable neural interface.

RIS/IRS research develops a closely related hardware vocabulary---phase and amplitude control, cascaded propagation, calibration, and beamforming---but usually targets communication and sensing objectives~\citep{wu2019joint,wu2020smart,liu2021ris}. RAOA adopts an operator viewpoint and trains a specific recurrent alternation end to end. The distinction is not that programmable surfaces can compute, but how state, task-dependent local fields, and repeated pass types are organized and evaluated.

\subsection{Physical neural networks}
Coherent photonic meshes and diffractive optical neural networks showed that propagation can realize trained linear transformations and inference pipelines~\citep{shen2017nanophotonic,lin2018d2nn}. Deep physical neural networks broaden the substrate class and demonstrate hybrid physical/digital training across optical, electronic, and mechanical systems~\citep{wright2022physical,li2024mechanical}. The review by \citet{momeni2025training} organizes physical-neural-network training around the constraints and opportunities of the underlying physics.

Our work follows this hardware-aware learning philosophy but focuses on a radio-specific recurrent operator. The static/conditioned/feedback hierarchy isolates where input dependence and state enter; RAOA then composes those ingredients through alternating problem/mixer passes. This is an architectural contribution rather than a claim that physical neural networks or programmable propagation are new.

\subsection{Nonlinear computation from linear or weakly nonlinear wave physics}
Several recent results show that nonlinear input-output computation need not require a strongly nonlinear material. \citet{wanjura2024nonlinear} encode data in physical parameters of a linear scattering system; \citet{yildirim2024nonlinear} synthesize higher-order nonlinear transforms through repeated data modulation and multiple scattering; \citet{xia2024recurrent} obtain nonlinear optical encoding through recurrent linear scattering; and \citet{hughes2019wave} establish a mapping between wave dynamics and analog recurrent neural networks.

These works are directly relevant to RAOA. Our contribution is not the observation that a linear wave equation can participate in nonlinear computation. Instead, RAOA makes the recurrent state and the two pass roles explicit, then evaluates that structure on controlled exact objectives and pretrained language models.

\subsection{Programmable-metasurface expressivity and depth}
The closest recent expressivity study is \citet{hammami2026expressivity}. Using a physics-consistent D-band programmable-metasurface cavity, they separately vary encoding nonlinearity, structural nonlinearity, and WPNN depth. Importantly, their shared-weight architecture shows that increasing WPNN depth can improve expressivity without increasing trainable-weight count. We therefore do \emph{not} claim parameter-independent wave depth as a first or unique result.

The overlap also clarifies the distinction. Their study asks how structural input encoding, mutual coupling, and WPNN layer depth govern scalar-regression expressivity. RAOA instead defines temporal feedback recurrence over one latent state with explicit problem and mixer passes. We use tied two-scalar controls to study depth on exact Max-Cut, 3-SAT, hypergraph Max-Cut, and partition objectives with known optima, including cross-depth execution without reoptimization, then reuse the same operator structure as a zero-initialized residual adapter in pretrained Transformers. The two studies are complementary: one characterizes programmable-metasurface WPNN expressivity in a physics-consistent cavity model; the other studies a task-structured recurrent operator across optimization and foundation-model regimes.

\subsection{Stacked intelligent metasurfaces and electromagnetic neural networks}
Stacked intelligent metasurfaces (SIMs) form electromagnetic neural networks by cascading trainable metasurface layers. They have been studied for task-oriented semantic communication, wireless processing, sensing, and edge inference~\citep{huang2025simsemcom,liu2025sim,an2026wave,stylianopoulos2026minn}. \citet{stylianopoulos2025ota} explicitly optimize RIS/SIM-enabled wireless propagation as one or more hidden DNN layers for end-to-end edge inference. These works provide strong radio-frequency precedents for learned wave-domain inference; RAOA does not claim that role for the wireless channel as new.

Most SIM/MINN depth in these studies is spatial or cascaded: signals traverse multiple trainable metasurface layers, possibly embedded in a wider end-to-end network. RAOA's evaluated depth is temporal/feedback recurrence over a latent state, with alternating problem/mixer semantics and pass count that can be increased while tying the learned controls. The axes are complementary: a future hardware realization could combine stacked spatial surfaces with RAOA-style temporal recurrence.

\subsection{Structured iterative networks and neural operators}
Algorithm unrolling interprets iterations of an optimization algorithm as trainable network layers~\citep{gregor2010lista,monga2021unrolling}. Neural operators learn mappings between function spaces and provide a broader language for learned transformations beyond finite-dimensional feed-forward blocks~\citep{kovachki2023neuraloperator,li2021fno}. RAOA adopts the structured-iteration perspective: each layer is a pair of problem/mixer passes over a shared latent state rather than a generic dense transformation.

The alternation is reminiscent of QAOA~\citep{farhi2014qaoa}, which alternates problem and mixer operators. The analogy is architectural only. RAOA uses a classical real-valued state, tanh saturation, modeled radio passes, and standard backpropagation. We make no quantum equivalence or quantum-advantage claim.

\subsection{Parameter-efficient adaptation of foundation models}
LoRA adds low-rank residual updates to frozen weights and is a standard parameter-efficient adaptation baseline~\citep{hu2022lora}. Our shallow MLP baseline supplies a generic nonlinear bottleneck at similar parameter scale. RAOA is compared to both because its contribution is structural: we ask whether an adapter constrained to repeated problem/mixer passes can approach the quality of an ordinary digital nonlinear adapter. The results do not support a universal ranking. LoRA is best on Qwen WikiText, MLP is best on Gemma and Phi WikiText and Phi GSM8K, while the 768-token Qwen and Gemma audit makes RAOA, MLP, and frozen reasoning accuracy nearly indistinguishable at paired example-bootstrap resolution. This mixed outcome is central to the paper's claim discipline.

\section{RAOA on Exact Discrete Optimization Problems}\label{app:discrete_optimization}

The main paper presents the four-problem modeled-depth comparison because it directly supports the central RAOA claim. This appendix supplies the exact encodings, instance statistics, complete depth table, and higher-order representation controls behind that result. The purpose is not to claim a general-purpose combinatorial optimizer; rather, we use small problems with exactly known optima to isolate how repeated RAOA passes behave across pairwise and higher-order objectives while the number of learned controls is held fixed.

\subsection{Objective encodings}
We work with binary variables $x_i\in\{0,1\}$ and the corresponding spins $s_i\in\{-1,+1\}$. Pairwise objectives are represented as QUBOs,
\begin{equation}
  E_{\rm QUBO}(x)=x^\top Qx+c,
\end{equation}
and are converted exactly to Ising form before the RAOA problem update. For weighted Max-Cut, the QUBO is chosen so that minimizing $E_{\rm cut}$ maximizes the cut weight. Number partition uses
\begin{equation}
  E_{\rm part}(x)=\left(\sum_i a_i(2x_i-1)\right)^2,
\end{equation}
so an exact zero corresponds to a balanced partition.

RAOA also accepts higher-order binary optimization (HUBO) objectives directly. For a SAT clause $C$, define the unsatisfied-literal factor
\begin{equation}
 u_\ell(x)=
 \begin{cases}
 1-x_i, & \ell=+i,\\
 x_i, & \ell=-i.
 \end{cases}
\end{equation}
The SAT energy is the number of unsatisfied clauses,
\begin{equation}
  E_{\rm SAT}(x)=\sum_{C}\prod_{\ell\in C}u_\ell(x),
\end{equation}
which is order $k$ for $k$-SAT. For hypergraph Max-Cut, an edge is penalized only when all of its vertices receive the same binary label:
\begin{equation}
 E_{\rm hcut}(x)=\sum_{e}w_e\left[\prod_{i\in e}x_i+\prod_{i\in e}(1-x_i)\right].
\end{equation}
Thus higher-order interactions can be presented natively to the RAOA local-field computation rather than first reduced to pairwise form.

The four fixed multi-seed instances are all small enough that the logical optimum can be verified exactly. Max-Cut-8 uses the fixed weighted graph in the checked-in experiment script; partition-8 uses weights $(1,2,3,4,6,7,9,10)$; 3-SAT-8 contains 12 fixed three-literal clauses; and hypergraph Max-Cut-8 contains eight fixed three- and four-vertex hyperedges. Table~\ref{tab:discrete_instances} reports the exact degeneracy of each optimum. The final column is the probability that a uniformly random binary assignment is already optimal, which helps contextualize the RAOA hit rates.

\begin{table}[h]
\centering
\caption{Exact small discrete-optimization instances used in the appendix study.}
\label{tab:discrete_instances}
\begin{tabular}{lrrrr}
\toprule
Problem & Order & Variables & Optimum degeneracy & Uniform optimum rate \\
\midrule
Max-Cut-8 & 2 & 8 & 2 & 0.78\% \\
Partition-8 & 2 & 8 & 10 & 3.91\% \\
3-SAT-8 & 3 & 8 & 36 & 14.06\% \\
Hypergraph Max-Cut-8 & 4 & 8 & 28 & 10.94\% \\
\bottomrule
\end{tabular}
\end{table}

\subsection{Fixed-control modeled-pass sweep}
To isolate modeled execution depth from parameter count, we share one $\gamma$ and one $\beta$ across every RAOA stage. Therefore depths $L\in\{1,2,4,8\}$ use exactly two learned scalar controls but $2L\in\{2,4,8,16\}$ modeled RF passes. The QUBO/HUBO coefficients are fixed data defining each problem instance; they are not learned controls. Each optimization uses Adam for 120 iterations at learning rate $0.04$, 64 continuous training replicas, 512 independent evaluation replicas, and no post-hoc local search. We repeat the complete optimization with five seeds. Optimum-hit probability is measured against the exact discrete optimum.

\begin{table*}[t]
\centering
\caption{Exact-optimum hit probability for tied-control RAOA (mean $\pm$ seed standard deviation, five seeds). The learned-control count is two in every cell; RF-pass counts are modeled.}
\label{tab:discrete_depth}
\begin{tabular}{lrrrr}
\toprule
Problem & 2 RF passes & 4 RF passes & 8 RF passes & 16 RF passes \\
\midrule
Max-Cut-8 & $0.071\pm0.017$ & $0.204\pm0.014$ & $0.433\pm0.034$ & $0.690\pm0.347$ \\
Partition-8 & $0.064\pm0.007$ & $0.025\pm0.012$ & $0.041\pm0.006$ & $0.008\pm0.005$ \\
3-SAT-8 & $0.617\pm0.020$ & $0.890\pm0.026$ & $0.980\pm0.013$ & $1.000\pm0.000$ \\
Hypergraph Max-Cut-8 & $0.352\pm0.021$ & $0.622\pm0.014$ & $0.856\pm0.011$ & $0.575\pm0.467$ \\
\bottomrule
\end{tabular}
\end{table*}

The problem dependence is substantial. Max-Cut-8 improves by almost an order of magnitude in mean optimum-hit probability from two to sixteen passes, although one of the five $L=8$ optimization runs collapses to zero success and produces the large final error bar. 3-SAT-8 is much more stable: the mean increases from $0.617$ at two passes to $0.980$ at eight passes and reaches $1.000$ for all five $L=8$ runs. Hypergraph Max-Cut improves consistently through eight passes, from $0.352$ to $0.856$, but the $L=8$ setting becomes bimodal across seeds: three optimizations reach approximately $0.94$--$0.96$ success while two collapse near zero. Partition remains the counterexample already highlighted in the main paper; added depth does not help and eventually reduces success below the two-pass setting.

These results refine the statement that modeled pass depth can be a computational resource. Repeating the same two learned controls can markedly increase success, and this holds for both pairwise Max-Cut and native higher-order 3-SAT. However, deeper RAOA also changes the optimization landscape: at sixteen passes, Max-Cut and hypergraph Max-Cut can become seed-sensitive, while partition remains non-monotonic. Thus the useful resource is \emph{available recurrent depth}, not a guarantee that every additional pass improves every objective.

\subsection{Native higher-order objectives versus quadratization}

The main paper emphasizes independent-instance depth generality; Table~\ref{tab:discrete_quadratization} retains the fixed-instance native-versus-quadratized values and resource expansion.

A standard way to execute a HUBO on a pairwise optimizer is to introduce ancilla variables and quadratize the objective. RAOA does not require that reduction because the problem layer evaluates the derivative of a sparse spin polynomial directly. We therefore compare native-HUBO RAOA with its exactly quadratized QUBO under the same two-control parameterization, optimizer, depth, and evaluation protocol. The comparison is intentionally at a fixed RAOA budget rather than a matched total-state dimension: quadratization itself changes the optimization problem by increasing that dimension.

For 3-SAT-8, exact quadratization expands 8 logical variables to 17 executed variables (9 ancillas); for hypergraph Max-Cut-8 it expands 8 to 16 (8 ancillas). Table~\ref{tab:discrete_quadratization} reports the full pass sweep.

\begin{table*}[t]
\centering
\caption{Native higher-order RAOA versus exact quadratization. Success is exact-optimum hit probability, mean $\pm$ seed standard deviation over five seeds.}
\label{tab:discrete_quadratization}
\small
\setlength{\tabcolsep}{5pt}
\begin{tabular}{llrrrr}
\toprule
Problem & Representation & Vars. & Anc. & 4 passes & 8 passes \\
\midrule
3-SAT-8 & Native HUBO & 8 & 0 & $0.896\pm0.013$ & $0.989\pm0.003$ \\
3-SAT-8 & Quadratized QUBO & 17 & 9 & $0.052\pm0.011$ & $0.161\pm0.023$ \\
Hypergraph Max-Cut-8 & Native HUBO & 8 & 0 & $0.603\pm0.012$ & $0.830\pm0.008$ \\
Hypergraph Max-Cut-8 & Quadratized QUBO & 16 & 8 & $0.056\pm0.006$ & $0.119\pm0.011$ \\
\bottomrule
\end{tabular}
\end{table*}

At eight modeled RF passes, native 3-SAT reaches $0.989$ success versus $0.161$ after quadratization, and native hypergraph Max-Cut reaches $0.830$ versus $0.119$. These experiments do \emph{not} establish that native higher-order execution always dominates quadratization: the quadratized problems contain roughly twice as many executed variables, and other penalty choices or optimization budgets could change the comparison. What they show is narrower and relevant to RAOA: when higher-order local fields are available directly, avoiding ancilla expansion can substantially simplify the fixed-budget search on these exact small instances.

For completeness, at sixteen passes the native 3-SAT result remains essentially perfect ($0.9996\pm0.0008$), while the quadratized version is unstable ($0.084\pm0.108$). Hypergraph Max-Cut at sixteen passes is itself seed-sensitive in native form ($0.576\pm0.466$), reinforcing that native higher-order representation removes ancilla overhead but does not remove the general depth-optimization instability observed above.

\paragraph{Claim boundary.}
All results in this section are synthetic classical RAOA mechanism studies with exactly known small optima. The higher-order problem layer is a differentiable phenomenological local-field computation. We do not claim measured nonlinear RF hardware for three- and four-body interactions, quantum equivalence, asymptotic advantage over digital combinatorial solvers, or hardware speed or energy improvement from these experiments.

\clearpage
\section{Conceptual System Views}\label{app:visuals}
\begin{figure*}[h]
\centering
\includegraphics[width=0.98\textwidth]{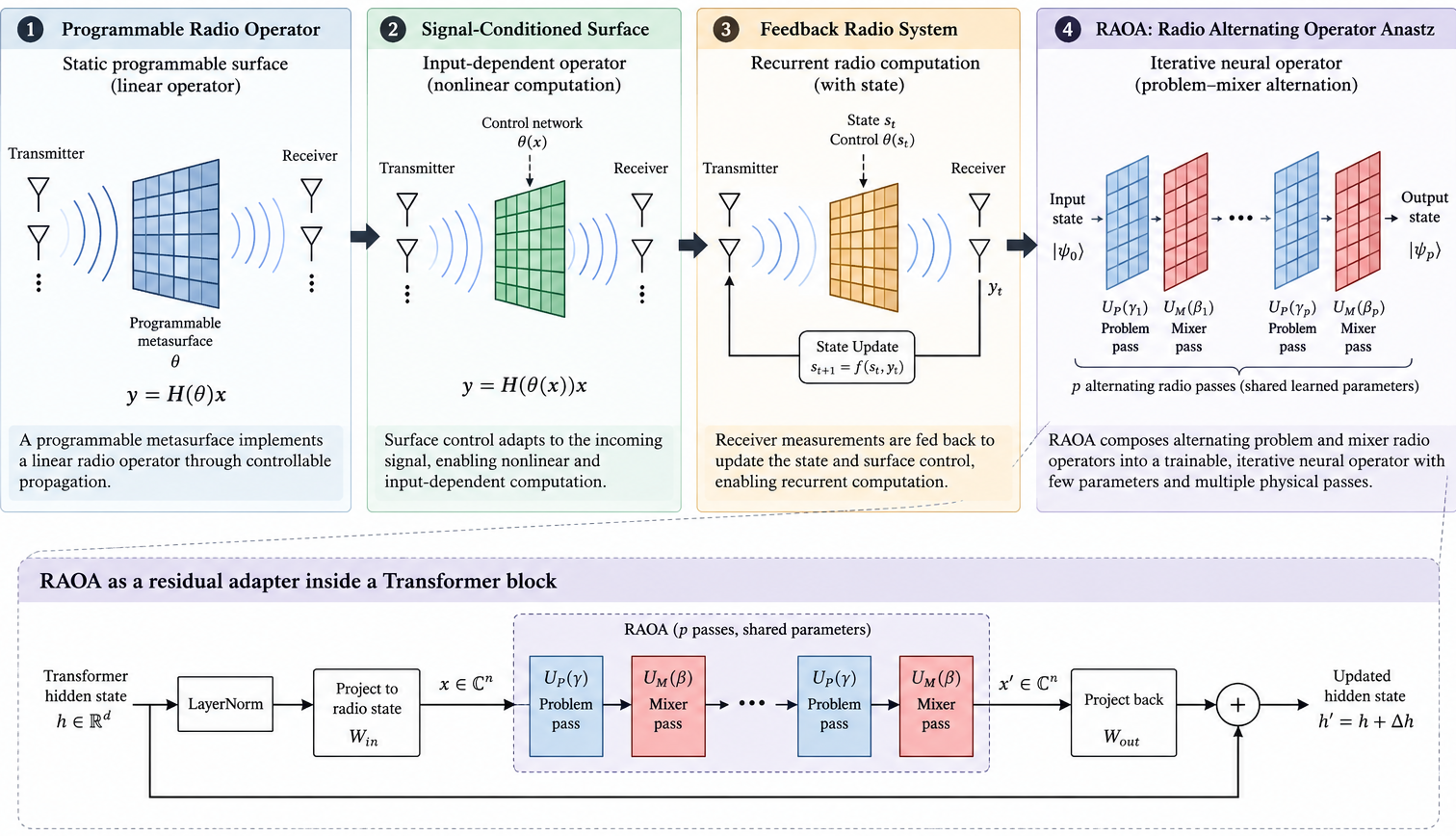}
\caption{\textbf{AI-generated conceptual system illustration.} This appendix-only raster image provides qualitative intuition for programmable propagation, signal-conditioned control, feedback, repeated alternating passes, and Transformer insertion. It is not experimental evidence. The image is pre-cropped to remove the original project-title band and generated caption; all quantitative plots in this paper are produced programmatically from frozen numerical results.}
\label{fig:ai_system_overview}
\end{figure*}

\section{Expressivity Hierarchy: Formulation and Controlled Evidence}\label{app:hierarchy}

Figure~\ref{fig:ladder} in the main paper summarizes this hierarchy and its two controlled mechanism checks. The derivations below make each transition explicit.

\subsection{Static programmable propagation is linear in the transmitted signal}
For fixed propagation $H$ and fixed programmable surface response $\Phi$,
\begin{equation}
  \mathcal T_{\Phi}(x)=H\Phi x
\end{equation}
is linear: $\mathcal T_{\Phi}(\alpha x_1+\beta x_2)=\alpha\mathcal T_{\Phi}(x_1)+\beta\mathcal T_{\Phi}(x_2)$. Reconfiguration changes which linear map is implemented but does not change this property for one programmed state.

\subsection{Signal-conditioned control creates nonlinear end-to-end maps}
Let the surface state depend on the input, $\Phi=\Phi(x)$. Then
\begin{equation}
  \mathcal T(x)=H\Phi(x)x
\end{equation}
is generally nonlinear. Even the scalar case makes the mechanism explicit: choosing $\phi(x)=x^2$ yields $y=\phi(x)x=x^3$. More generally, if a control network produces $g(x)$ and the surface performs elementwise multiplication, the physical layer contains products $g_i(x)x_j$, which can generate higher-order interactions despite linear propagation at each frozen control state.

Our controlled cubic task targets $y=x^3$. The best scalar linear surface reaches MSE $2.3715\times10^{-2}$, whereas the signal-conditioned reference reaches numerically zero error in the ideal software model. Under the preregistered severe synthetic hardware profile, signal-conditioned MSE remains approximately $2.4\times10^{-3}$ versus $2.57\times10^{-2}$ for the static linear surface. These values are mechanism tests, not hardware measurements.

\subsection{Feedback introduces state and temporal computation}
A feedback radio system evolves as
\begin{equation}
  z_{t+1}=F_\theta(z_t,x_t),\qquad y_t=G_\theta(z_t,x_t).
\end{equation}
For a linear state update and linear readout, the output is a linear functional of the input history. Consequently, a delayed multiplicative target such as $x_{t-d_1}x_{t-d_2}$ cannot in general be represented by the linear loop alone. In the controlled reference, linear feedback obtains roughly chance-level delayed-product accuracy ($49.3\%$), while nonlinear feedback reaches $100\%$ under the ideal software model. The purpose of this example is to verify the additional temporal/nonlinear mechanism rather than to claim a general reservoir-computing benchmark.

\subsection{RAOA adds iterative modeled-pass depth}
RAOA alternates a problem-dependent update and a mixing update. A tied program can reuse only two learned controls while increasing modeled RF passes, separating modeled recurrent depth from parameter count. Appendix~\ref{app:discrete_optimization} expands this mechanism across Max-Cut, partition, 3-SAT, and hypergraph Max-Cut with five optimization seeds and exact optima; Sec.~\ref{app:depth} then isolates the Max-Cut and partition scale-control question. The resulting picture is deliberately non-monotonic: extra passes can add substantial computation, but whether optimization can exploit that depth depends on the objective and seed.

\section{RAOA and Foundation-Model Adapter Details}\label{app:method}

\subsection{General RAOA state update}
For recurrent state $z^{(\ell)}\in\mathbb{R}^m$, context $c$, and differentiable energy $E_\psi(z;c)$, define
\begin{align}
  u^{(\ell)} &= \nabla_z E_\psi(z^{(\ell)};c),\\
  s^{(\ell)} &= \frac{1}{m}\sum_{i=1}^{m}|u_i^{(\ell)}|+\epsilon,\\
  \tilde z^{(\ell)} &= \tanh\!\left(z^{(\ell)}-\sin\gamma_\ell\frac{u^{(\ell)}}{s^{(\ell)}}\right),\\
  z^{(\ell+1)} &= \tanh\!\left(\mathcal M_{\beta_\ell}(\tilde z^{(\ell)})\right),\qquad
  \mathcal M_\beta(z)=\cos\beta\,z+\sin\beta\,Mz.
\end{align}
The mixer $M$ is a fixed Hadamard-derived orthogonal matrix. The implementation detaches $s^{(\ell)}$ in the backward pass. One stage is one problem pass plus one mixer pass, hence $2L$ modeled RF passes for depth $L$.

A general multilinear energy is
\begin{equation}
  E_\psi(z;c)=\sum_{S\in\mathcal S}w_S(c)\prod_{i\in S}z_i,
  \qquad
  u_i=\sum_{S\ni i}w_S(c)\prod_{j\in S\setminus\{i\}}z_j.
\end{equation}
The foundation-model specialization truncates this at order two,
\begin{equation}
  E_\psi(z;c)=\tfrac12z^\top Jz+b(c)^\top z,\qquad b(c)=W_{\rm in}c,
\end{equation}
with symmetric zero-diagonal $J$, giving $u=Jz+b(c)$. Exact QUBO and HUBO studies instead use the known multilinear spin polynomial. Three- and four-body terms therefore enter through products in the local field rather than through ancilla variables; this remains a phenomenological software representation rather than evidence of arbitrary many-body RF realization.

\subsection{Normalization and positive global rescaling}
The detached local normalization explains the Hamiltonian-rescaling control. If $E'_\psi=qE_\psi$ for $q>0$, then $u'=qu$ and $s'=qs$ up to the additive $\epsilon$ term, so $u'/s'\approx u/s$. A positive global coefficient scale should therefore leave the ideal RAOA trajectory essentially unchanged. Sec.~\ref{app:depth} verifies this directly. RAOA uses ordinary unrolled autograd; it is neither an implicit-differentiation method nor a claim of quantum-QAOA equivalence.

\subsection{Residual Transformer insertion and parameter count}
For hidden state $h$, let $c=\operatorname{LN}(h)$. The adapter uses $b(c)=W_{\rm in}c$ both as the context field and to initialize $z^{(0)}=\tanh(b(c))$, executes $L$ RAOA stages, and maps the final state back to model width:
\begin{equation}
   \Delta h=W_{\rm out}z^{(L)}(c),\qquad
   f_{\rm adapt}(h)=f_{\rm base}(h)+a\Delta h,\qquad a|_{t=0}=0.
\end{equation}
Hence the modified model equals the pretrained model exactly at insertion time. With biased input and output projections, coupling stored as an unconstrained $m\times m$ parameter, and $L$ learned angle pairs, one adapter has
\begin{equation}
  P_{\rm core}=2dm+m^2+d+m+2L,\qquad P_{\rm wrapped}=P_{\rm core}+1
\end{equation}
trainable scalar parameters. The mixer in the final foundation-model experiments is fixed rather than learned.

\subsection{Modeled RF impairments}
For control angle $\alpha\in\{\gamma_\ell,\beta_\ell\}$,
\begin{equation}
  \hat\alpha=Q_b(\alpha+\varepsilon_\alpha),\quad
  \varepsilon_\alpha\sim\mathcal N(0,\sigma_\alpha^2),\qquad
  \hat z=z+\varepsilon_z,\quad \varepsilon_z\sim\mathcal N(0,\sigma_z^2I).
\end{equation}
During domain-randomized training, angle quantization uses a straight-through estimator. The RAOA-specific profile values are:
\begin{table}[h]
\centering
\caption{Modeled RAOA impairment profiles. These are synthetic perturbation levels, not device measurements.}
\label{tab:hardware_profiles}
\begin{tabular}{lrrr}
\toprule
Profile & Angle bits & $\sigma_\alpha$ (rad) & $\sigma_z$\\
\midrule
Ideal & -- & 0 & 0\\
Mild & 8 & 0.005 & 0.002\\
Moderate & 6 & 0.020 & 0.010\\
Severe & 4 & 0.060 & 0.040\\
\bottomrule
\end{tabular}
\end{table}

\section{Theoretical Properties of RAOA}\label{app:theory}
This section isolates several structural properties of RAOA. The first two propositions analyze an \emph{unsaturated fixed-normalization surrogate}: elementwise $\tanh$ is replaced by the identity and the scalar field normalization is held fixed. This surrogate is used only to expose composition; it is not a global equivalence to the full nonlinear operator. The final propositions state properties of the pre-saturation problem step and of the full bounded recurrence under explicit assumptions. The alternating problem/mixer viewpoint follows the QAOA/alternating-operator template~\citep{farhi2014qaoa,hadfield2019alternating}, while the power-sequence terminology below is standard in Krylov methods~\citep{saad2003iterative}. Algorithm unrolling provides a broader classical perspective on repeated learned or structured iterations~\citep{monga2021unrolling}.

\begin{proposition}[Affine power expansion at fixed tied controls]\label{prop:affine_power}
Let
\[
E(z)=\tfrac12 z^\top Jz+b^\top z,
\qquad B_\beta=\cos\beta I+\sin\beta M,
\]
with symmetric $J$. In the unsaturated surrogate with fixed normalization $s>0$, let $\eta=\sin\gamma/s$ and tie $(\gamma,\beta)$ across stages. Then
\[
 z^{(\ell+1)}=Az^{(\ell)}+c,
 \qquad A=B_\beta(I-\eta J),\qquad c=-\eta B_\beta b,
\]
and after $L$ stages
\begin{equation}
 z^{(L)}=A^Lz^{(0)}+\sum_{k=0}^{L-1}A^k c.
 \label{eq:app_affine_power}
\end{equation}
Consequently, as depth varies, the trajectory is generated by the block-Krylov-type span of powers of the same matrix $A$ acting on $z^{(0)}$ and $c$. Increasing $L$ can therefore enlarge the available composed computation without increasing the two tied controls, although strict subspace growth is not guaranteed once the relevant minimal polynomial saturates.
\end{proposition}

\begin{proof}
The quadratic gradient is $\nabla E(z)=Jz+b$. Removing saturation and using fixed $s$ gives the problem step $(I-\eta J)z-\eta b$; applying the linear mixer $B_\beta$ yields $z^{(\ell+1)}=Az^{(\ell)}+c$. Eq.~\ref{eq:app_affine_power} follows by induction: it is immediate for $L=1$, and substituting the depth-$L$ expression into $z^{(L+1)}=Az^{(L)}+c$ produces $A^{L+1}z^{(0)}+\sum_{k=0}^{L}A^k c$. The span statement follows directly. This is the same matrix-power structure underlying standard Krylov constructions~\citep{saad2003iterative}.
\end{proof}

The mixer can change the generated operator algebra even with tied controls. For example,
$A^2=B_\beta(I-\eta J)B_\beta(I-\eta J)$ contains terms involving $B_\beta J B_\beta$ and $B_\beta J B_\beta J$; when $M$ and $J$ do not commute, expanding $B_\beta$ introduces distinct words such as $MJ$, $JM$, and $MJM$. This does not imply that a particular mixer is always beneficial, consistent with the objective-dependent mixer ablation in Appendix~\ref{app:mechanism_ablation}.

\begin{proposition}[Dynamic-field degree growth versus a frozen field]\label{prop:degree_growth}
Let $E:\mathbb R^m\to\mathbb R$ be a polynomial of total degree $d\ge2$ and consider the unsaturated stage with either no normalization or fixed normalization
\[
 G(z)=B\bigl(z-\eta\nabla E(z)\bigr),
\]
where $B$ is linear and $\eta$ is a fixed scalar. Then
\[
 \deg G\le d-1,
 \qquad
 \deg\bigl(G^{\circ L}\bigr)\le(d-1)^L.
\]
Equality may fail because of coefficient cancellation or degeneracy; the bound states the maximum compositional degree available from dynamic field re-evaluation. If instead the initial field is frozen, $g_0=\nabla E(z^{(0)})$, and
\[
 z^{(\ell+1)}=B\bigl(z^{(\ell)}-\eta g_0\bigr),
\]
then
\begin{equation}
 z^{(L)}=B^Lz^{(0)}-\eta\sum_{j=1}^{L}B^j\nabla E(z^{(0)}),
 \label{eq:frozen_field_closed_form}
\end{equation}
and $z^{(L)}$, viewed as a polynomial function of $z^{(0)}$, has degree at most $d-1$ for every $L$.
\end{proposition}

\begin{proof}
Because differentiation lowers total polynomial degree by at least one, $\deg\nabla E\le d-1$. A linear map does not increase total degree, hence $\deg G\le\max\{1,d-1\}=d-1$. For polynomial maps $f$ and $g$, $\deg(f\circ g)\le\deg(f)\deg(g)$; applying this relation inductively gives $\deg(G^{\circ L})\le(d-1)^L$. For a frozen field, repeated substitution gives Eq.~\ref{eq:frozen_field_closed_form}. Both $B^Lz^{(0)}$ and every $B^j\nabla E(z^{(0)})$ have degree at most $d-1$, so the degree cannot grow with $L$.
\end{proof}

For quadratic energies ($d=2$), Proposition~\ref{prop:degree_growth} reduces to linear or affine power composition, as in Proposition~\ref{prop:affine_power}. For cubic and quartic energies the bound permits degrees up to $2^L$ and $3^L$, respectively, before saturation. The full RAOA normalization and $\tanh$ make the exact map non-polynomial, so this result should be read as a structural surrogate rather than an exact degree statement for the implemented operator. Its empirical counterpart is the large degradation observed when the local field is frozen rather than recomputed.

\begin{proposition}[Problem-step scale invariance and local descent]\label{prop:problem_step}
Define $\mu(u)=m^{-1}\sum_i|u_i|$ and $\mathcal N_\epsilon(u)=u/(\mu(u)+\epsilon)$. For any $a>0$,
\begin{equation}
 \mathcal N_\epsilon(au)-\mathcal N_\epsilon(u)
 =\frac{\epsilon(a-1)}{(a\mu(u)+\epsilon)(\mu(u)+\epsilon)}\,u.
 \label{eq:scale_invariance_exact}
\end{equation}
Hence positive global energy rescaling is exactly removed when $\epsilon=0$ and differs from exact invariance only through the displayed $O(\epsilon)$ term when $\epsilon>0$.

Separately, suppose $E$ has $L_E$-Lipschitz gradient on a region containing $z$ and $y=z-\eta\nabla E(z)$. Then the pre-saturation problem step satisfies
\begin{equation}
 E(y)\le E(z)-\eta\left(1-\frac{L_E\eta}{2}\right)\|\nabla E(z)\|_2^2.
 \label{eq:descent_bound}
\end{equation}
Thus $0<\eta<2/L_E$ is sufficient for strict descent whenever $\nabla E(z)\ne0$. In RAOA, $\eta=\sin\gamma/s(z)$ for the current problem step. No monotonic-energy claim is made for the subsequent $\tanh$ or mixer pass.
\end{proposition}

\begin{proof}
For $a>0$, $\mu(au)=a\mu(u)$; subtracting the two normalized vectors gives Eq.~\ref{eq:scale_invariance_exact}. For the descent statement, the standard smoothness inequality for an $L_E$-Lipschitz gradient gives
\[
 E(y)\le E(z)+\nabla E(z)^\top(y-z)+\tfrac{L_E}{2}\|y-z\|_2^2.
\]
Substituting $y-z=-\eta\nabla E(z)$ yields Eq.~\ref{eq:descent_bound}. This is the usual descent-lemma argument for smooth first-order methods~\citep{nesterov2004convex}.
\end{proof}

\begin{proposition}[Bounded recurrence and finite-depth operator error]\label{prop:bounded_error}
Let $F$ denote one full RAOA stage, including the final elementwise $\tanh$. For every finite input, $F(z)\in(-1,1)^m$, so after one stage the recurrent state lies in the compact cube $\mathcal D=[-1,1]^m$. Let $\widehat F:\mathcal D\to\mathcal D$ be an approximate/compiled stage and assume, in a chosen norm,
\[
 \sup_{z\in\mathcal D}\|\widehat F(z)-F(z)\|\le\delta,
 \qquad
 \|F(x)-F(y)\|\le K\|x-y\|
 \quad\forall x,y\in\mathcal D.
\]
Starting both recurrences from the same state, after $L$ stages
\begin{equation}
 \|\widehat z^{(L)}-z^{(L)}\|
 \le \delta\sum_{k=0}^{L-1}K^k.
 \label{eq:operator_error_bound}
\end{equation}
In particular, the bound is $L\delta$ for $K=1$ and at most $\delta/(1-K)$ for any $K<1$.
\end{proposition}

\begin{proof}
Elementwise $\tanh$ maps every finite scalar to $(-1,1)$, establishing boundedness. Let $e_\ell=\|\widehat z^{(\ell)}-z^{(\ell)}\|$. Then
\[
 e_{\ell+1}
 \le\|\widehat F(\widehat z^{(\ell)})-F(\widehat z^{(\ell)})\|
   +\|F(\widehat z^{(\ell)})-F(z^{(\ell)})\|
 \le\delta+Ke_\ell.
\]
With $e_0=0$, induction gives Eq.~\ref{eq:operator_error_bound}. The stated special cases follow from the geometric series. The proposition does not assert that the implemented stage is contractive ($K<1$); it only quantifies how a known stage-wise approximation error can accumulate over finite depth.
\end{proof}

Taken together, Propositions~\ref{prop:affine_power}--\ref{prop:bounded_error} explain four empirical observations without turning them into universal guarantees: tied controls can still yield richer repeated composition; dynamic field re-evaluation can add nonlinear compositional order; the normalized problem step suppresses arbitrary positive global objective scale; and imperfect physical operators may remain useful when their stage error is sufficiently controlled over the execution horizon.

\section{Experimental Protocols and Full Results}\label{app:experiments}

\subsection{WikiText architecture transfer}
The WikiText protocol uses 65,536 training tokens, 65,536 validation tokens for hyperparameter selection, and full-test evaluation. Selected schedules are frozen before the test split is evaluated. Each trainable method uses seeds 1--3. The number of scored test tokens is 296,545 for Qwen3-4B, 289,941 for Gemma-3-4B-it, and 284,734 for Phi-4-mini-instruct.

\begin{table}[h]
\centering
\caption{Validation-selected WikiText schedules used for full-test confirmation.}
\label{tab:wikitext_schedules}
\begin{tabular}{llrr}
\toprule
Model & Family & Steps & Learning rate\\
\midrule
Qwen3-4B & LoRA & 300 & $10^{-4}$\\
          & RAOA & 300 & $10^{-3}$\\
          & MLP & 300 & $10^{-3}$\\
Gemma-3-4B-it & LoRA & 300 & $10^{-4}$\\
          & RAOA & 300 & $10^{-3}$\\
          & MLP & 300 & $3\times10^{-4}$\\
Phi-4-mini & LoRA & 300 & $10^{-4}$\\
          & RAOA & 300 & $10^{-3}$\\
          & MLP & 300 & $10^{-3}$\\
\bottomrule
\end{tabular}
\end{table}

\begin{table}[h]
\centering
\caption{Full-test WikiText NLL (mean $\pm$ seed standard deviation). Lower is better.}
\label{tab:wikitext_full}
\begin{tabular}{lrrrr}
\toprule
Model & Frozen & LoRA & RAOA & Shallow MLP\\
\midrule
Qwen3-4B & 3.3007 & $\mathbf{2.6426}\pm0.0016$ & $2.6851\pm0.0040$ & $2.6596\pm0.0023$\\
Gemma-3-4B-it & 5.9833 & $2.9046\pm0.0051$ & $2.9009\pm0.0055$ & $\mathbf{2.8763}\pm0.0061$\\
Phi-4-mini & 3.1010 & $2.6232\pm0.0006$ & $2.6503\pm0.0058$ & $\mathbf{2.6210}\pm0.0023$\\
\bottomrule
\end{tabular}
\end{table}

Relative to the frozen baselines, RAOA captures 96.0\%, 99.2\%, and 93.9\% of the shallow MLP NLL gain on Qwen, Gemma, and Phi, respectively. The average gain capture is approximately 96.4\%. These percentages measure recovery of MLP improvement, not statistical equivalence.

\subsection{Qwen3-4B validation-only RAOA depth diagnostic}
To test whether the exact-objective depth trend automatically transfers to language modeling, we froze the already selected Qwen RAOA schedule (300 steps, learning rate $10^{-3}$, all MLP blocks, state width 64) and changed only depth. Three-seed WikiText validation NLL is $2.58531\pm0.00459$ for $L=1$, $2.58781\pm0.00329$ for $L=2$, and $2.58415\pm0.00062$ for $L=4$. The full range is only $0.00366$ NLL and is non-monotonic. We therefore treat the primary $L=1$ foundation-model comparison as an operator-transfer experiment and do not claim a language-model depth benefit.

\subsection{Matched digital recurrent Qwen control}\label{app:recurrent_baseline}
To test whether Qwen WikiText transfer can be explained by generic tied recurrence rather than RAOA's specific problem/mixer structure, we evaluate the existing depth-2 digital recurrent bottleneck in all text MLP blocks and choose its latent width to match the trainable-parameter budget of primary RAOA ($L=1$, state width 64). The resulting recurrent model has 12,038,436 trainable parameters versus 12,038,544 for RAOA, a difference of only 108 parameters ($0.0009\%$). This comparison matches trainable parameter count, not computational cost.

The schedule is selected without the test split. A six-cell WikiText validation screen uses 65,536 training tokens, 65,536 validation tokens, seed 1, steps $\{100,300\}$, and learning rates $\{10^{-4},3\times10^{-4},10^{-3}\}$. The minimum validation NLL, $2.57828$, occurs at 300 steps and learning rate $10^{-3}$. That schedule is selected before test evaluation. Three-seed full-test confirmation at that frozen schedule gives
\[
  \mathrm{NLL}_{\rm recurrent}=2.67267\pm0.00289,
  \qquad \Delta\mathrm{NLL}=-0.62800
\]
from the common frozen Qwen baseline $3.30067$. For context, the same full-test protocol gives LoRA $2.6426\pm0.0016$, shallow MLP $2.6596\pm0.0023$, and RAOA $2.6851\pm0.0040$. The recurrent control therefore recovers about $98.0\%$ of the shallow MLP improvement over frozen and is numerically $0.0124$ NLL below RAOA, but we do not attach a significance claim to this four-way comparison. The result supplies a matched recurrent digital control while narrowing the foundation-model interpretation: RAOA transfers effectively, but this experiment does not show that its radio-alternating structure is uniquely advantageous over generic recurrence.

\subsection{GSM8K free-running reasoning}
The GSM8K selection stage trains on examples $[0,256)$ and selects hyperparameters using a disjoint holdout beginning at example 256 from the training split. Final schedules are frozen before any use of the official test split. The canonical full-test evaluation uses all 1,319 test examples, a common 384-token generation cap, model-native stop tokens, one deterministic frozen run, and three seeds for each trainable method. Cap-limited outputs remain in exact-match scoring.

\begin{table}[h]
\centering
\caption{Holdout-selected GSM8K schedules used for full-test confirmation.}
\label{tab:gsm8k_schedules}
\begin{tabular}{llrr}
\toprule
Model & Family & Steps & Learning rate\\
\midrule
Qwen3-4B & LoRA & 100 & $10^{-5}$\\
          & RAOA & 25 & $10^{-5}$\\
          & MLP & 50 & $10^{-5}$\\
Gemma-3-4B-it & LoRA & 10 & $3\times10^{-6}$\\
          & RAOA & 5 & $3\times10^{-5}$\\
          & MLP & 10 & $3\times10^{-5}$\\
Phi-4-mini & LoRA & 50 & $10^{-5}$\\
          & RAOA & 50 & $10^{-4}$\\
          & MLP & 100 & $10^{-4}$\\
\bottomrule
\end{tabular}
\end{table}

\begin{table}[h]
\centering
\caption{Canonical 384-token full-test GSM8K exact-match accuracy (\%, mean $\pm$ seed standard deviation for trainable methods).}
\label{tab:gsm8k_full}
\begin{tabular}{lrrrr}
\toprule
Model & Frozen & LoRA & RAOA & Shallow MLP\\
\midrule
Qwen3-4B & 81.35 & $76.55\pm1.70$ & $\mathbf{81.88}\pm0.12$ & $81.65\pm0.16$\\
Gemma-3-4B-it & 73.77 & $\mathbf{74.75}\pm1.29$ & $73.97\pm0.19$ & $73.21\pm0.26$\\
Phi-4-mini & 79.23 & $83.04\pm0.23$ & $79.45\pm0.39$ & $\mathbf{83.73}\pm0.53$\\
\bottomrule
\end{tabular}
\end{table}

The canonical cap-hit rates are non-negligible for Qwen and Gemma: roughly 10.7--14.3\% and 11.9--16.3\%, respectively. We therefore preregistered a sensitivity-only audit that re-evaluates the already frozen Qwen and Gemma schedules at 768 tokens, with no hyperparameter selection or retraining based on the longer-cap test results. Per-example predictions are retained. Phi is not rerun because its adapted 384-token cap-hit rates are mostly below 2\%.

\begin{table}[h]
\centering
\caption{768-token GSM8K sensitivity audit for Qwen and Gemma. Accuracy is full-test EM (\%); cap is max-token-limit-hit rate (\%).}
\label{tab:gsm8k_768}
\begin{tabular}{llrrr}
\toprule
Model & Family & EM & Seed std. & Cap hit\\
\midrule
Qwen3-4B & Frozen & 87.95 & -- & 0.38\\
          & LoRA & 81.65 & 1.96 & 3.23\\
          & RAOA & 87.95 & 0.49 & 0.53\\
          & MLP & 88.05 & 0.25 & 0.63\\
Gemma-3-4B-it & Frozen & 81.12 & -- & 4.17\\
          & LoRA & 79.83 & 0.98 & 3.36\\
          & RAOA & 81.43 & 0.27 & 4.25\\
          & MLP & 81.25 & 0.29 & 4.40\\
\bottomrule
\end{tabular}
\end{table}

For paired inference, correctness is first averaged across the three training seeds within each trainable family for each of the 1,319 examples, then examples are resampled 10,000 times with a fixed bootstrap seed. Table~\ref{tab:gsm8k_pairwise} reports accuracy differences in percentage points. These intervals characterize example sampling conditional on the completed seed ensemble; they are not a full retraining-level uncertainty analysis.

\begin{table}[h]
\centering
\caption{Paired example-bootstrap differences at 768 tokens (A$-$B, percentage points; 95\% CI).}
\label{tab:gsm8k_pairwise}
\begin{tabular}{llrr}
\toprule
Model & Pair & Delta & 95\% CI\\
\midrule
Qwen3-4B & RAOA$-$Frozen & 0.00 & $[-0.78,\,0.78]$\\
          & RAOA$-$MLP & -0.10 & $[-0.66,\,0.43]$\\
          & LoRA$-$Frozen & -6.29 & $[-8.11,\,-4.50]$\\
Gemma-3-4B-it & RAOA$-$Frozen & +0.30 & $[-0.53,\,1.16]$\\
          & RAOA$-$MLP & +0.18 & $[-0.45,\,0.81]$\\
          & LoRA$-$Frozen & -1.29 & $[-2.68,\,0.13]$\\
\bottomrule
\end{tabular}
\end{table}

The long-cap audit changes the interpretation of the small canonical differences rather than establishing a stable ranking: Qwen RAOA, MLP, and frozen converge near 88\%, and Gemma RAOA, MLP, and frozen converge near 81\%. Qwen LoRA remains 6.29 points below frozen with a paired interval excluding zero. For completeness, RAOA$-$LoRA is $+6.29$ points on Qwen (95\% CI $[4.60,7.96]$) and $+1.59$ points on Gemma ($[0.38,2.86]$); MLP$-$LoRA is $+6.39$ ($[4.68,8.04]$) and $+1.42$ ($[0.15,2.65]$), respectively. The full-test 384-token results remain valid fixed-cap measurements and are retained above as protocol history and sensitivity evidence.

\subsection{Gemma RAOA placement and coverage}\label{app:placement}
The placement ablation uses one common schedule (50 steps, learning rate $3\times10^{-5}$), three seeds, and all 1,319 GSM8K test examples. This is a fixed-schedule mechanism study, not the primary tuned comparison.

Figure~\ref{fig:structure_robustness} (left) in the main paper shows the placement curve; Table~\ref{tab:placement_full} preserves the exact values and parameter counts.

\begin{table}[h]
\centering
\caption{Full-test Gemma placement results.}
\label{tab:placement_full}
\begin{tabular}{lrrr}
\toprule
Placement & Blocks & Params (M) & GSM8K EM (\%)\\
\midrule
All & 34 & 11.370 & $74.30\pm0.38$\\
Every 2 & 17 & 5.685 & $74.00\pm0.51$\\
Every 4 & 9 & 3.010 & $73.69\pm0.54$\\
Last 2 & 2 & 0.669 & $73.87\pm0.19$\\
First+last & 2 & 0.669 & $73.21\pm0.50$\\
First 2 & 2 & 0.669 & $73.01\pm0.27$\\
\bottomrule
\end{tabular}
\end{table}

A 128-example holdout pilot had ranked first-two placement highest. The full-test evaluation reverses that conclusion, so the pilot is retained only as a sampling-sensitivity example and is not used as a final claim.

\subsection{Modeled-pass depth and Hamiltonian normalization}\label{app:depth}
The depth control uses Max-Cut-8 and partition-8, depths $L\in\{1,2,4,8\}$, two tied learned controls at every depth, five optimization seeds, 64 training replicas, 512 evaluation replicas, and ideal and moderate synthetic profiles. We compare the raw Hamiltonian to coefficient-RMS normalization:
\begin{equation}
  (J,h,c_0)\mapsto (J/s,h/s,c_0/s),\qquad s>0,
\end{equation}
where $s$ is the RMS of nonzero unique Ising coupling and field coefficients. Since this is a positive global rescaling,
\begin{equation}
  \arg\min_z E(z)=\arg\min_z E(z)/s.
\end{equation}

Across all 16 combinations of problem, profile, and depth, 15 raw-versus-normalized mean success values are exactly identical. The only nonzero difference is $1.95\times10^{-4}$ at partition depth 8 under the moderate profile. This rules out simple global coefficient scale as a sufficient explanation for the partition-depth anomaly.

\subsection{Independent-instance depth and cross-depth control transfer}\label{app:random_depth}
The independent-instance depth study uses 12 independently generated instances for each of Max-Cut, partition, 3-SAT, and hypergraph Max-Cut at $n\in\{8,12\}$. Five optimizer seeds are averaged within each problem instance before the cross-instance mean and standard deviation are computed. Every train depth learns only one shared $\gamma$ and one shared $\beta$. Each learned pair is then evaluated at every requested depth without reoptimization.

\begin{table*}[h]
\centering
\caption{Random-instance exact-optimum success when controls are trained only at two passes and reused unchanged at deeper execution depths. Values are mean $\pm$ standard deviation across 12 independent instances after averaging optimizer seeds within instance.}
\label{tab:random_depth_transfer}
\small
\setlength{\tabcolsep}{5pt}
\begin{tabular}{llrrrr}
\toprule
Family & $n$ & 2 passes & 4 passes & 8 passes & 16 passes\\
\midrule
3-SAT & 8 & $0.543\pm0.148$ & $0.862\pm0.114$ & $0.971\pm0.038$ & $0.997\pm0.008$\\
3-SAT & 12 & $0.367\pm0.104$ & $0.745\pm0.117$ & $0.923\pm0.067$ & $0.980\pm0.026$\\
Hypergraph Max-Cut & 8 & $0.586\pm0.062$ & $0.817\pm0.050$ & $0.945\pm0.026$ & $0.995\pm0.003$\\
Hypergraph Max-Cut & 12 & $0.441\pm0.044$ & $0.725\pm0.048$ & $0.914\pm0.030$ & $0.984\pm0.014$\\
Max-Cut & 8 & $0.0785\pm0.0273$ & $0.175\pm0.069$ & $0.291\pm0.130$ & $0.390\pm0.226$\\
Max-Cut & 12 & $0.0158\pm0.0086$ & $0.0520\pm0.0324$ & $0.111\pm0.069$ & $0.175\pm0.131$\\
Partition & 8 & $0.0356\pm0.0162$ & $0.0465\pm0.0241$ & $0.0565\pm0.0362$ & $0.0733\pm0.0560$\\
Partition & 12 & $0.0304\pm0.0111$ & $0.0419\pm0.0168$ & $0.0489\pm0.0478$ & $0.0594\pm0.126$\\
\bottomrule
\end{tabular}
\end{table*}

Matched-depth optimization is also broadly positive but exposes an optimization caveat. Relative to matched $L=1$, all 12 3-SAT instances improve at $L=2,4,8$ for both sizes. At $L=8$, hypergraph Max-Cut improves on 11/12 instances for $n=8$ and 12/12 for $n=12$; Max-Cut improves on 10/12 and 12/12, respectively. Partition is mixed: at $n=8$, $L=8$ improves 2/12 and worsens 10/12; at $n=12$, it improves 5/12 and worsens 7/12. The contrast with Table~\ref{tab:random_depth_transfer} indicates that repeated execution can remain useful even when depth-specific optimization fails to exploit it reliably.

\paragraph{Tied versus untied controls.}
A separate five-seed fixed-instance sweep compares two tied controls with $2L$ untied per-stage controls under ideal, moderate, and severe profiles. Under ideal evaluation, the four-problem portfolio mean is $0.5701$ for tied versus $0.5689$ for untied at eight passes (2 versus 8 learned controls), and $0.5659$ versus $0.5323$ at sixteen passes (2 versus 16 controls). Individual objectives differ substantially---for example, untied controls help deep Max-Cut but hurt deep hypergraph Max-Cut---so the result is not an argument that tying is universally better. It supports the narrower statement that the observed depth effect is not generally explained by simply adding per-stage control parameters.

\subsection{Mechanism ablation}\label{app:mechanism_ablation}
The preregistered mechanism control fixes depth at $L=4$ (eight modeled RF passes), uses five seeds, and applies the same optimization schedule to every variant without variant-specific rescue tuning. Table~\ref{tab:mechanism_ablation} shows that freezing the local field at its initial state is consistently destructive, whereas mixer removal or replacement is objective-dependent.

\begin{table}[h]
\centering
\caption{Mechanism ablation at eight modeled RF passes. Values are mean exact-optimum success over five seeds. ``Frozen field'' computes the local field once on the initial state and reuses it at every problem pass.}
\label{tab:mechanism_ablation}
\small
\setlength{\tabcolsep}{4.2pt}
\begin{tabular}{lrrrrr}
\toprule
Problem & Full & $\beta{=}0$ & Frozen field & Identity & Rand. orth.\\
\midrule
3-SAT & 0.985 & 0.978 & 0.483 & 0.980 & 0.987\\
Hypergraph Max-Cut & 0.850 & 0.853 & 0.040 & 0.879 & 0.856\\
Max-Cut & 0.430 & 0.319 & 0.060 & 0.320 & 0.317\\
Partition & 0.0356 & 0.105 & 0.0062 & 0.111 & 0.076\\
\bottomrule
\end{tabular}
\end{table}

The field-recomputation effect is large on 3-SAT, hypergraph Max-Cut, and Max-Cut; the four-problem mean falls from $0.575$ for full RAOA to $0.147$ with a frozen initial field. By contrast, there is no universal mixer ordering. The Hadamard mixer is materially useful on Max-Cut, but beta-zero, identity, or random-orthogonal mixing is nearly neutral on 3-SAT and hypergraph Max-Cut and can improve partition under this matched-depth protocol. We therefore attribute the most consistent recurrent mechanism to re-evaluating the problem field on the evolving state, while treating cross-coordinate mixing as an objective-dependent resource rather than a universally necessary ingredient.

\subsection{RIS propagation model, geometry, and physical interpretation}\label{app:ris_model}
The constrained realizability study instantiates the programmable aperture as a passive narrowband reconfigurable intelligent surface (RIS). At the communication-model level, a phase-programmed RIS is commonly represented by a diagonal reflection matrix embedded between transmitter-to-surface and surface-to-receiver channels~\citep{wu2019joint,liu2021ris}. Physics-based analyses refine this abstraction by deriving aperture- and distance-dependent path loss from physical optics or Green-function formulations and by distinguishing near- and far-field regimes~\citep{ozdogan2020pathloss,danufane2021green}. Free-space RIS path-loss models of this type have also been compared with fabricated-surface measurements in anechoic environments~\citep{tang2021rispathloss}. We use the same first-order geometric structure, but deliberately stop short of an absolute-power or full-wave hardware model.

\paragraph{Physical RIS embodiment.}
A practical reflective RIS is a planar or nearly planar array of electromagnetically coupled unit cells whose local impedance is electronically configured so that the reflected field acquires a programmable complex coefficient. Implementations commonly use resonant patches or meta-atoms loaded by switches or tunable reactances; a digital control network sets the element states. For example, a fabricated $16\times10$ surface at $5.8\,$GHz used patch elements with RF PIN diodes and a control circuit for beam steering~\citep{kashyap2021ris5g}, while Tang et al.~\citep{tang2021rispathloss} measured several fabricated metasurfaces to validate RIS free-space path-loss models. Our coefficient $e^{j\phi_{p,k}}$ is therefore the ideal continuous, unit-magnitude limit of such programmable reflection states. Real cells generally have finite phase resolution, loss, element patterns, and phase--amplitude coupling; these effects are deliberately separated from the operator-shape experiment below.

\paragraph{Scalar spherical-wave channel.}
Let $t_a\in\mathbb{R}^3$ and $r_b\in\mathbb{R}^3$ denote transmit and receive mode locations, and let $s_{p,k}$ denote element $k$ of panel $p$. At carrier $f_c$ with wavelength $\lambda=c/f_c$, the modeled complex field transfer between two points separated by $d$ is
\begin{equation}
  g(d)=\frac{\lambda}{4\pi d}\exp\!\left(-j\frac{2\pi d}{\lambda}\right).
  \label{eq:ris_green_scalar}
\end{equation}
For point sets $A$ and $B$, $G_{B\leftarrow A}$ contains the corresponding pairwise values of $g$. Panel $p$ has ideal unit-amplitude coefficients
\begin{equation}
  \Phi_p=\operatorname{diag}\!\left(e^{j\phi_{p,1}},\ldots,e^{j\phi_{p,K}}\right),
\end{equation}
and contributes the single-bounce operator
\begin{equation}
  H_p(\phi_p)=G_{R\leftarrow S_p}\,\Phi_p\,G_{S_p\leftarrow T}.
  \label{eq:ris_single_bounce}
\end{equation}
With $P$ panels, the experiment coherently sums independent parallel paths,
\begin{equation}
  H_{\rm RIS}(\phi)=\sum_{p=1}^{P}H_p(\phi_p),
  \label{eq:ris_parallel_panels}
\end{equation}
with no direct Tx--Rx term and no RIS-to-RIS cascade. Thus the model retains exact point-to-point spherical phase and $1/d$ field decay instead of reducing each aperture to a plane-wave steering vector. This distinction is relevant here: at $10\,$GHz, $\lambda\simeq3.0\,$cm, while the panel's $0.80\,$m transverse span is about $27\lambda$ and the modeled separations are sub-meter to meter scale. Near-/far-field dependence and aperture scaling are known to matter in RIS links~\citep{ozdogan2020pathloss,tang2021rispathloss,danufane2021green}.

\paragraph{Geometry used in the reported sweep.}
The pairwise $n=8$ problems use nine Tx and nine Rx modes because one fixed $+1$ carrier augments the eight recurrent coordinates. For every geometry seed, positions are drawn once and reused for both the problem and mixer programs. If $\mathcal U(q)$ denotes an independent coordinate-wise uniform perturbation in $[-q_d/2,q_d/2]$, then
\begin{align}
 t_a &= (0,0,0)+\mathcal U(0.01,0.12,0.12),\nonumber\\
 r_b &= (1.20,0.10,0.04)+\mathcal U(0.01,0.12,0.12),\label{eq:ris_tx_rx_geometry}\\
 s_{p,k} &= c_p+\mathcal U(0.02,0.80,0.80),\nonumber
\end{align}
all in meters, where each panel has $K=128$ elements and
\begin{equation}
 c_p=\left(0.20+\frac{0.80p}{P+1},\ y_p,\ 0.03(p-1)\right),\quad
 y_p=\begin{cases}-0.18,&p\ \text{odd},\\+0.20,&p\ \text{even}.
 \end{cases}
 \label{eq:ris_panel_centers}
\end{equation}
We sweep $P\in\{1,2,4\}$ and three geometry seeds. The thin $x$ span and larger $y$--$z$ span make each element cloud near-planar, but the points are intentionally irregular rather than a manufacturable $\lambda/2$ lattice; this avoids a special regular-array rank pathology and makes the study a geometry-constrained operator test rather than a specific fabricated-panel design. Surface amplitudes and coordinates are fixed. Only the continuous phases are optimized, giving 128, 256, or 512 programmable variables per phase program.

\paragraph{Compilation and receiver readout.}
The problem and mixer are two phase programs on the same geometry. For target real operator $T$, the compiler first fixes a positive gain scale $\alpha=\|H_{\rm RIS}(\phi_0)\|_F/\|T\|_F$ and then minimizes normalized complex Frobenius error,
\begin{equation}
 \phi^*=\arg\min_{\phi}\frac{\|H_{\rm RIS}(\phi)-\alpha T\|_F^2}{\|\alpha T\|_F^2}.
 \label{eq:ris_compile_objective}
\end{equation}
RAOA subsequently uses the in-phase component $\operatorname{Re}[H_{\rm RIS}(\phi^*)/\alpha]$; the residual quadrature fraction is reported separately. Because $\alpha$ removes overall gain, these experiments test whether the \emph{operator shape} can be approximated under passive phase-only propagation, not whether a given aperture closes an RF link budget.

\paragraph{What is and is not physically modeled.}
The model above is a standard first-order cascaded RIS abstraction with elementwise spherical propagation. They capture geometry-dependent phase, free-space distance loss, coherent panel addition, and passive phase-only control. More complete RIS models additionally account for effective element aperture and radiation pattern, polarization, mutual coupling, phase-amplitude coupling, finite-bit control, dispersion and loss, edge effects, and the environment-dependent direct/multipath channel~\citep{ozdogan2020pathloss,tang2021rispathloss,danufane2021green}. Our experiment omits those effects, as well as walls, clutter, blockage, channel estimation, synchronization, and feedback latency. In particular, ideal unit-amplitude point elements plus global gain normalization should not be interpreted as a calibrated radar-cross-section or received-power model; RIS path-loss scaling is sensitive to aperture and propagation regime, and simplistic interpretations can be misleading~\citep{bjornson2020myths}. The result is therefore a constrained \emph{realizability} check between an unconstrained dense operator and a measured system: more physical than an arbitrary matrix, but not a substitute for full-wave simulation or hardware calibration.

\subsection{Passive phase-only propagation realization}\label{app:phase_only}
Using the channel and geometry in Appendix~\ref{app:ris_model}, for pairwise Max-Cut-8 and partition-8 we compile the problem and mixer operators into explicit free-space transfer matrices formed by coherent sums of passive unit-amplitude panels. Each panel has 128 phase elements; the sweep uses 1, 2, or 4 panels and three independent geometry seeds. No trainable amplitude control, dense digital precoder or combiner, or direct path is used. A positive scalar gain normalization absorbs overall path loss. This study does not cover native three- and four-body local fields.

\begin{table}[h]
\centering
\caption{Normalized Frobenius error of phase-only compiled operators (mean $\pm$ standard deviation across three geometry seeds).}
\label{tab:phase_operator_error}
\small
\begin{tabular}{llrrr}
\toprule
Problem & Operator & 1 panel & 2 panels & 4 panels\\
\midrule
Max-Cut & Problem & $0.687\pm0.016$ & $0.519\pm0.021$ & $0.357\pm0.092$\\
Max-Cut & Mixer & $0.758\pm0.048$ & $0.573\pm0.055$ & $0.406\pm0.092$\\
Partition & Problem & $0.711\pm0.034$ & $0.535\pm0.059$ & $0.348\pm0.142$\\
Partition & Mixer & $0.758\pm0.048$ & $0.573\pm0.055$ & $0.406\pm0.092$\\
\bottomrule
\end{tabular}
\end{table}

The approximation error decreases consistently as passive resources increase, while the imaginary-energy fraction falls from roughly 0.44--0.48 with one panel to roughly 0.21--0.22 with four panels. Downstream Max-Cut also improves with panel count. With four panels and reoptimization of only the same two $\gamma,\beta$ controls, success is $0.056$, $0.144$, $0.305$, and $0.487$ at 2, 4, 8, and 16 modeled passes, compared with ideal values $0.068$, $0.195$, $0.428$, and $0.686$. Reusing ideal controls unchanged gives $0.054$, $0.145$, $0.301$, and $0.410$. At eight and sixteen passes, two-control reoptimization therefore retains about 71\% of ideal mean success, although geometry variability becomes large at sixteen passes. Partition does not admit a simple retention interpretation because its ideal matched-depth baseline is itself non-monotonic; several compiled settings exceed that unstable baseline. We treat that as another warning that operator approximation and optimization landscape interact.

\subsection{Modeled RF robustness}\label{app:hardware}
Qwen3-4B RAOA is trained in two modes: ideal only, or uniform per-step domain randomization over the four profiles in Table~\ref{tab:hardware_profiles}. Both use 300 steps, learning rate $10^{-3}$, 65,536 WikiText training tokens, and seeds 1--3. Every trained model is evaluated on the full WikiText test split under all four profiles.

Figure~\ref{fig:hardware_appendix} and Table~\ref{tab:hardware_full} give the complete synthetic-perturbation result. This control remains in the appendix because the main physical-scope figure emphasizes constrained phase-only propagation.

\begin{figure}[h]
\centering
\includegraphics[width=0.72\linewidth]{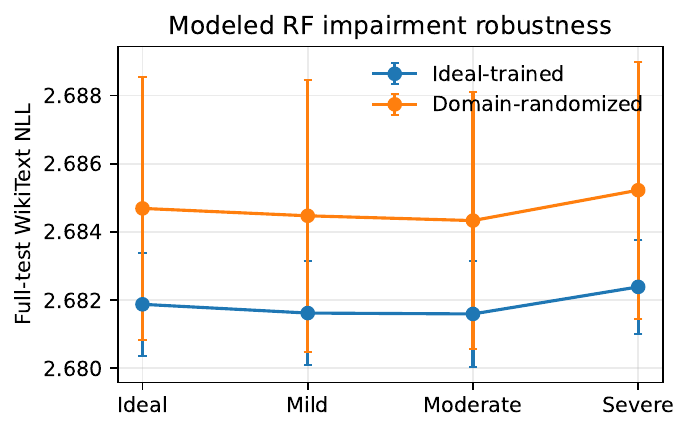}
\caption{\textbf{Synthetic RAOA perturbation sensitivity on Qwen3-4B WikiText.} Ideal-trained RAOA is nearly invariant through the preregistered severe angle/state profile; domain-randomized training is consistently slightly worse. These are modeled perturbations, not measured RF impairments.}
\label{fig:hardware_appendix}
\end{figure}

\begin{table}[h]
\centering
\caption{Modeled-impairment WikiText NLL (mean $\pm$ seed standard deviation).}
\label{tab:hardware_full}
\begin{tabular}{lrr}
\toprule
Evaluation profile & Ideal-trained & Domain-randomized\\
\midrule
Ideal & $2.68188\pm0.00151$ & $2.68469\pm0.00385$\\
Mild & $2.68162\pm0.00153$ & $2.68447\pm0.00399$\\
Moderate & $2.68159\pm0.00156$ & $2.68433\pm0.00377$\\
Severe & $2.68239\pm0.00137$ & $2.68523\pm0.00377$\\
\bottomrule
\end{tabular}
\end{table}

The severe profile changes ideal-trained NLL by only $+0.000511$ relative to ideal evaluation. Domain-randomized training is worse than ideal-trained RAOA by about $0.00274$--$0.00285$ NLL under every profile. The evidence therefore supports near-invariance to this modeled perturbation family, not a benefit from impairment-aware training.

\section{Detailed Limitations and Claim Boundaries}\label{app:limitations}
The main-text discussion is intentionally concise. This section records the assumptions, negative results, and boundaries that should accompany the positive findings.

\paragraph{Constrained propagation is not measured hardware.}
Reported RF-pass counts are modeled operator executions rather than measured wall-clock passes through a calibrated RF loop. The phase-only experiment replaces unconstrained dense effective matrices with explicit passive free-space transfer functions and therefore narrows the realizability gap, but it remains a simulation. It does not include measured antenna responses, metasurface calibration, synchronization, converter nonidealities, feedback delay, oscillator drift, mutual coupling, or end-to-end control latency. Accordingly, the paper makes no measured claim about energy efficiency, throughput, latency, or closed-loop robustness.

\paragraph{The phase-only compiler is deliberately restricted.}
The realizability study uses 1, 2, or 4 coherent passive panels with 128 phase-only elements each, fixed free-space propagation, no direct path, and no trainable dense digital precoder or combiner. Panels contribute as a coherent sum of independent single-bounce paths rather than a cascaded stack. For the 8-spin pairwise tasks, one fixed $+1$ bias carrier gives a 9-mode augmented target, so the one-, two-, and four-panel settings use 128, 256, and 512 programmable phase variables; this study tests feasibility under one explicit constraint, not physical resource efficiency. More panels reduce operator error and recover more Max-Cut behavior, but residual normalized matrix error remains substantial and deep executions show geometry variability. These results establish that useful computation can survive one concrete propagation constraint; they do not show that arbitrary target operators are synthesizable with a practical aperture budget or under a measured channel.

\paragraph{Higher-order physical realization remains open.}
Native third- and fourth-order local fields are executed directly in software and outperform the tested exact quadratizations under a fixed small RAOA pass-and-control budget. The comparison isolates representation and ancilla burden, not a universal solver advantage over all reductions or digital optimization methods. The phase-only compiler currently covers only pairwise operators, so the paper does not demonstrate a measured or constrained-wave realization of arbitrary many-body fields. Signal-conditioned propagation, feedback, auxiliary latent variables, or other nonlinear RF mechanisms may provide such a route, but that remains future work.

\paragraph{Exact-objective scaling is limited.}
The independent-instance study uses four objective families at $n\in\{8,12\}$, with exact optima available by enumeration. This is sufficient to test whether recurrent execution depth can add computation at fixed learned-control count, but it is not a scaling study. Success probabilities, optimization landscapes, and the usefulness of tied controls may change substantially at larger $n$, denser constraints, or different objective ensembles. No asymptotic complexity or advantage over mature combinatorial solvers is claimed.

\paragraph{Depth is useful but not uniformly trainable.}
Cross-depth reuse shows that a shallow-trained two-control program can often benefit from additional execution, yet matched-depth optimization is not monotonic. Partition remains inconsistent, and Max-Cut and hypergraph settings can become seed-sensitive at large depth. The tied-versus-untied comparison also shows that adding stage-specific controls does not systematically fix this behavior. Thus execution depth, parameterization, and optimization landscape should be treated as separate design variables rather than as a monotonic capacity axis.

\paragraph{Mixer utility is objective-dependent.}
Freezing the initial local field causes the dominant degradation across the tested objectives, supporting state-dependent field re-evaluation as the robust recurrent mechanism. In contrast, removing or replacing the Hadamard mixer is nearly neutral on 3-SAT and hypergraph Max-Cut, harmful on Max-Cut, and beneficial on the low partition baseline. We therefore do not claim that the chosen Hadamard mixer is universally necessary or optimal, nor do we perform a post-hoc mixer search after observing these results.

\paragraph{Theoretical results are structural, not global convergence guarantees.}
Appendix~\ref{app:theory} analyzes explicit surrogate regimes. The affine matrix-power result assumes an unsaturated quadratic update with fixed normalization, and the polynomial-degree result studies the corresponding unsaturated multilinear map. The local descent inequality applies to the pre-saturation problem step, not to the complete problem/mixer stage, and the compiled-operator error bound does not assume that the implemented RAOA stage is contractive. These propositions explain mechanisms and finite-depth composition; they do not establish global convergence, monotonic objective decrease, or optimality of deep RAOA.

\paragraph{Foundation-model evidence tests transfer more than depth.}
The primary Qwen3-4B, Gemma-3-4B-it, and Phi-4-mini-instruct experiments use a single RAOA stage inserted digitally as a residual adapter. They show that the operator can transfer across pretrained Transformer families at closely matched trainable-parameter budgets, not that a radio realization has been executed inside those models. A frozen-schedule Qwen validation diagnostic is essentially flat across $L=1,2,4$, so the exact-objective depth result should not be extrapolated to language models without additional evidence. The matched Qwen depth-2 digital recurrent baseline reaches $2.6727\pm0.0029$ full-test NLL, numerically between the shallow MLP and RAOA; accordingly, we also do not claim that RAOA's particular alternating structure is uniquely responsible for the foundation-model adaptation gain.

\paragraph{Parameter matching is not compute or hardware matching.}
LoRA, RAOA, and the bottleneck MLP are closely matched in trainable parameter count, which controls one important adaptation budget. Their sequential operation count, activation traffic, digital MAC cost, and prospective RF resource requirements are not matched. Likewise, the phase-only panel count is a physical-resource proxy rather than a measured energy or latency budget. The paper therefore does not claim end-to-end efficiency superiority over digital adapters.

\paragraph{Evaluation and statistical scope.}
Discrete-objective conclusions aggregate 12 independent instances per family-size setting after averaging five optimizer seeds within each instance, while foundation-model trainable methods use three seeds. These sample sizes support the reported empirical trends but do not establish universality across problem distributions or pretrained models. GSM8K retains cap-limited outputs in exact-match scoring and includes a preregistered 768-token sensitivity audit for the high-cap-hit Qwen and Gemma cases. Its paired intervals resample examples after averaging correctness across completed training seeds, so they measure example-sampling uncertainty conditional on those runs rather than full retraining uncertainty. The paper consequently emphasizes replicated directional effects and explicit negative results rather than fine-grained ranking claims.

\paragraph{Negative findings define the supported region.}
Several preregistered controls remain negative or mixed: matched-depth partition optimization is unstable; the Hadamard mixer is not universally best; deep exact settings can show seed or geometry sensitivity; the foundation-model depth diagnostic is flat under the frozen schedule; an early-layer placement pattern from a small Gemma holdout does not replicate on the full test set; and domain-randomized training does not help under the tested synthetic perturbation family. These outcomes are retained because they delimit where RAOA currently has evidence and where stronger claims would require new experiments or hardware.

\end{document}